\documentclass{article}

\PassOptionsToPackage{numbers, compress}{natbib}

\usepackage[preprint]{neurips_2021}

\usepackage[utf8]{inputenc} 
\usepackage[T1]{fontenc}    
\usepackage{hyperref}       
\usepackage{url}            
\usepackage{booktabs}       
\usepackage{amsfonts}       
\usepackage{nicefrac}       
\usepackage{microtype}      
\usepackage{xcolor}         
\usepackage{amsmath,bm,bbm}
\usepackage{amssymb}
\usepackage{amsthm}
\usepackage{graphicx}
\usepackage{multirow}
\usepackage{subcaption}
\usepackage{enumitem}

\usepackage{cases}
\usepackage{subeqnarray}

\newtheorem{theorem}{Theorem}
\newtheorem{lemma}{Lemma}

\makeatletter  
\newif\if@restonecol  
\makeatother

\usepackage[linesnumbered,ruled,vlined]{algorithm2e}
\usepackage{algpseudocode}  
\usepackage{amsmath}  
\algnewcommand\LeftComment[2]{%
	\hspace{#1\algindent}$\triangleright$ \eqparbox{COMMENT}{#2} \hfill %
}

\title{Confidence Adaptive Regularization for Deep Learning with Noisy Labels}
\author{%
  Yangdi Lu \\
  Department of Computing and Software\\
  McMaster University\\
  Hamilton, Canada \\
  \texttt{luy100@mcmaster.ca} \\
   \And
  Yang Bo \\
 Department of Computing and Software\\
 McMaster University\\
 Hamilton, Canada \\
 \texttt{boy2@mcmaster.ca} \\
 \And
 Wenbo He \\
 Department of Computing and Software\\
 McMaster University\\
 Hamilton, Canada \\
 \texttt{hew11@mcmaster.ca} \\
}
\begin{document}
\maketitle

\begin{abstract}

	Recent studies on the memorization effects of deep neural networks on noisy labels show that the networks first fit the correctly-labeled training samples before memorizing the mislabeled samples. Motivated by this early-learning phenomenon, we propose a novel method to prevent memorization of the mislabeled samples. Unlike the existing approaches which use the model output to identify or ignore the mislabeled samples, we introduce an indicator branch to the original model and enable the model to produce a confidence value for each sample. The confidence values are incorporated in our loss function which is learned to assign large confidence values to correctly-labeled samples and small confidence values to mislabeled samples. We also propose an auxiliary regularization term to further improve the robustness of the model. To improve the performance, we gradually correct the noisy labels with a well-designed target estimation strategy. We provide the theoretical analysis and conduct the experiments on synthetic and real-world datasets, demonstrating that our approach achieves comparable results to the state-of-the-art methods.
	
	

	  
\end{abstract}

\section{Introduction}
With the emergence of highly-curated datasets such as ImageNet \cite{deng2009imagenet} and CIFAR-10 \cite{krizhevsky2009learning}, deep neural networks have achieved remarkable performance on many classification tasks \cite{krizhevsky2012imat,noh2015learning,he2016deep}. However, it is extremely time-consuming and expensive to label a new large-scale dataset with high-quality annotations. Alternatively, we may obtain the dataset with lower quality annotations efficiently through online keywords queries \cite{li2017webvision} or crowdsourcing \cite{yu2018learning}, but \emph{noisy labels} are inevitably introduced consequently. Previous studies \cite{arpit2017closer,zhang2018understanding} demonstrate that noisy labels are problematic for overparameterized neural networks, resulting in overfitting and performance degradation. Therefore, it is essential to develop noise-robust algorithms for deep learning with noisy labels.

The authors of \cite{arpit2017closer,zhang2018understanding,li2020gradient,liu2020early} have observed that deep neural networks learn to correctly predict the true labels for all training samples during \emph{early learning} stage, and begin to make incorrect predictions in \emph{memorization} stage as it gradually memorizes the mislabeled samples (in Figure \ref{fig:memorization} (a) and (b)). In this paper, we introduce a novel regularization approach to prevent the memorization of mislabeled samples (in Figure \ref{fig:memorization} (c)). Our contributions are summarized as follows:

\begin{itemize}[leftmargin=0.3cm]
	\item We introduce an indicator branch to estimate the confidence of model prediction and propose a novel loss function called confidence adaptive loss (CAL) to exploit the early-learning phase. A high confidence value is likely to be associated with a clean sample and a low confidence value with a mislabeled sample. Then, we add an auxiliary regularization term forming a confidence adaptive regularization (CAR) to further segregate the mislabeled samples from the clean samples. We also develop a strategy to estimate the target probability instead of using the noisy labels directly, allowing the proposed model to suppress the influence of the mislabeled samples successfully.
	\item We theoretically analyze the gradients of the proposed loss functions and compare them with cross-entropy loss. We demonstrate that CAL and CAR have similar effects to existing regularization-based approaches. Both neutralize the influence of the mislabeled samples on the gradient, and ensure the contribution from correctly-labeled samples to the gradient remains dominant. We also prove the robustness of the auxiliary regularization term to label noise. 
 	\item We show that the proposed approach achieves comparable and even better performance to the state-of-the-art methods on four benchmarks with different types and levels of label noise. We also perform an ablation study to evaluate the influence of different components. 
\end{itemize}

\begin{figure}[t]
	\begin{center}
		\includegraphics[width=1.0\linewidth]{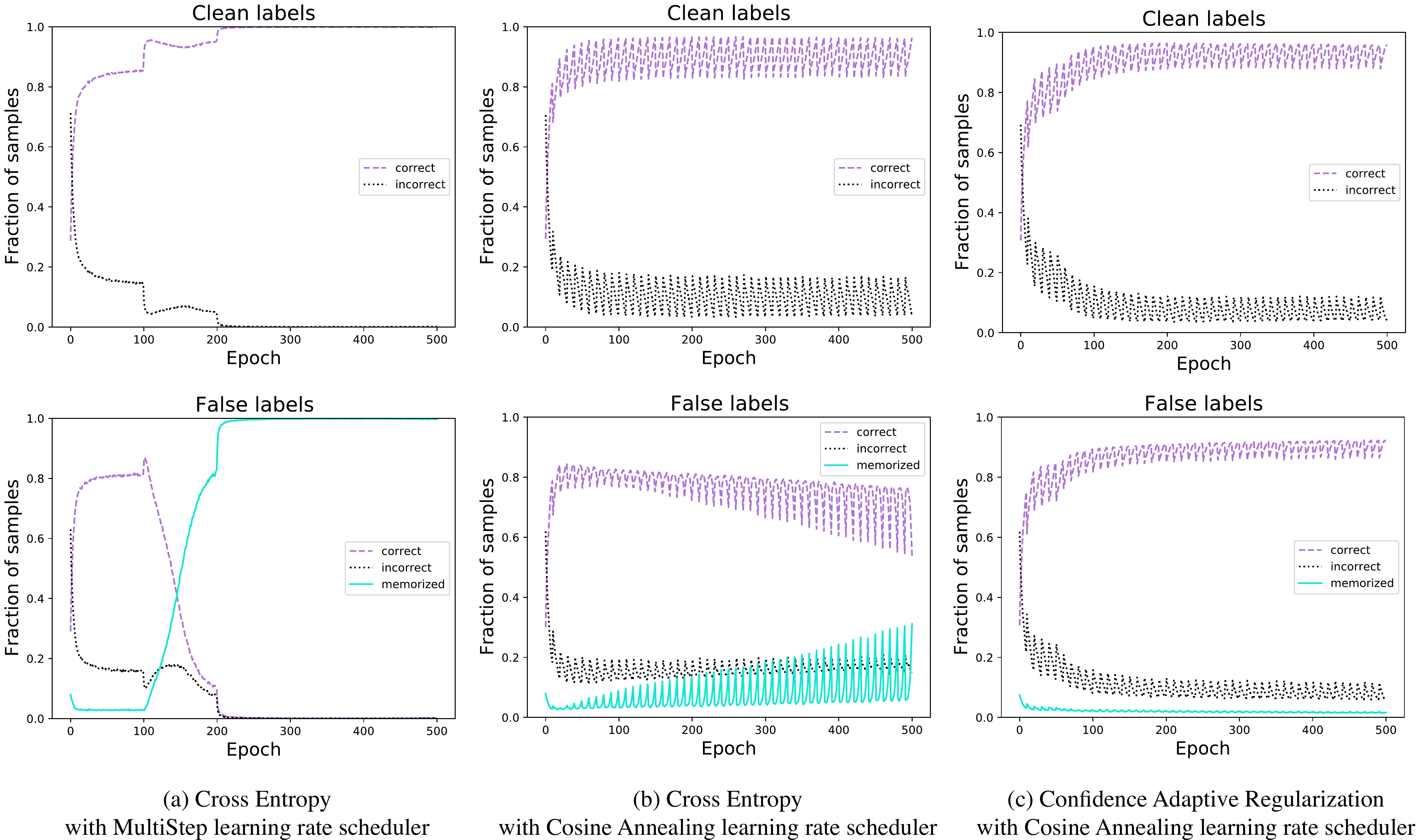}
	\end{center}
	\caption{We conduct the experiments on the CIFAR-10 dataset with 40\% symmetric label noise using ResNet34 \cite{he2016deep}. The top row shows the fraction of samples with clean labels that are predicted correctly (purple) and incorrectly (black). In contrast, the bottom row shows the fraction of samples with false labels that are predicted correctly (purple), \emph{memorized} (i.e. the prediction equals the false label, shown in blue), and incorrectly predicted as neither the true nor the labeled class (black). For samples with clean labels, all three models predict them correctly with the increasing of epochs. However, for false labels in (a), the model trained with cross-entropy loss first predicts the true labels correctly, but eventually memorizes the false labels. With the cosine annealing learning rate scheduler \cite{loshchilov2016sgdr} in (b), the model only slows down the speed of memorizing the false labels. However, our approach shown in (c) effectively prevents memorization, allowing the model to continue learning the correctly-labeled samples to attain high accuracy on samples with both clean and false labels.}
	\label{fig:memorization}
	\vspace{-0.5em}
\end{figure}

\section{Related work}
%
We briefly discuss the related noise-robust methods that do not require a set of clean training data (as opposed to \cite{xiao2015learning,vahdat2017toward,veit2017learning,li2017learning,hendrycks2018using,ren2018learning,lee2018cleannet}) and assume the label noise is instance-independent (as opposed to \cite{cheng2020learning,xia2020parts}). 

\textbf{Loss correction} These approaches focus on correcting the loss function explicitly by estimating the noise transition matrix \cite{goldberger2016training,patrini2017making,tanno2019learning,xia2019anchor}. \textbf{Robust loss functions} These studies develop loss functions that are robust to label noise, including $\mathcal{L}_\text{DMI}$ \cite{xu2019l_dmi}, MAE \cite{ghosh2017robust}, GCE \cite{zhang2018generalized}, IMAE \cite{wang2019imae}, SL \cite{wang2019symmetric} NCE \cite{ma2020normalized} and TCE \cite{feng2020can}. Above two categories of methods do not utilize the early learning phenomenon.  

\textbf{Sample selection} During the early learning stage, the samples with smaller loss values are more likely to be the correctly-labeled samples. Based on this observation, 
MentorNet \cite{jiang2017mentornet} pre-trains a mentor network for selecting small-loss samples to guide the training of the student network. Co-teaching related methods \cite{han2018co,yu2019does,wei2020combating,lu2021co} maintain two networks, and each network is trained on the small-loss samples selected by its peer network. However, their limitation is that they may eliminate numerous useful samples for robust learning. 
\textbf{Label correction} \cite{tanaka2018joint,yi2019probabilistic} replace the noisy labels with soft (i.e. model probability) or hard (i.e to one-hot vector) pseudo-labels. Bootstrap \cite{reed2014training} corrects the labels by using a convex combination of noisy labels and the model predictions. SAT \cite{huang2020self} weighs the sample with its maximal class probability in cross-entropy loss and corrects the labels with model predictions. \cite{pmlr-v97-arazo19a} weighs the clean and mislabeled samples by fitting a two-component Beta mixture model to loss values, and corrects the labels via convex combination as in \cite{reed2014training}. Similarly, DivideMix \cite{li2020dividemix} trains two networks to separate the clean and mislabeled samples via a two-component Gaussian mixture model, and further uses MixMatch \cite{berthelot2019mixmatch} to enhance the performance. \textbf{Regularization} \cite{li2020gradient} observes that when the model parameters remain close to the initialization, gradient descent \emph{implicitly} ignores the noisy labels. Based on this observation, they prove the gradient descent early stopping is an effective regularization to achieve robustness to label noise. \cite{hu2019simple} \emph{explicitly} adds the regularizer based on neural tangent kernel \cite{jacot2018neural} to limit the distance between the model parameters to
initialization. 
ELR \cite{liu2020early} estimates the target probability by temporal ensembling \cite{laine2016temporal} and adds a regularization term to cross entropy loss to avoid memorization. Other regularization techniques, such as mixup augmentation \cite{zhang2017mixup}, label smoothing \cite{pereyra2017regularizing} and weight averaging \cite{tarvainen2017mean}, can enhance the performance. 

Our approach is related to regularization and label correction. Compared with existing approaches \cite{hu2019simple,liu2020early}, where a regularization term in loss function is necessary to resist mislabeled samples, we propose a loss function CAL which \emph{implicitly} boosts the gradients of correctly labeled samples and diminishes the gradients of mislabeled samples. The auxiliary regularization term in our approach is an add-on component to further improve the performance in more challenging cases. We then propose a novel strategy to estimate the target and correct the noisy labels. In addition, our approach is simpler and yields comparable performance without applying other regularization techniques.

\section{Methodology}
This section presents a framework called confidence adaptive regularization (CAR) for robust learning from noisy labels. Our approach consists of three key elements: (1) We add an indicator branch to the original deep neural networks and estimate the confidence of the model predictions by exploiting the early-learning phenomenon through a confidence adaptive loss (CAL). (2) We propose an auxiliary regularization term explicitly designed to further separate the confidence of clean samples and mislabeled samples. (3) We estimate the target probabilities by incorporating the model predictions with noisy labels through a confidence-driven strategy. In addition, we analyze the gradients of CAL and CAR, and provide a theoretical guarantee for the noise-robust term in CAR.

\begin{figure}[t]
	\begin{center}
		\includegraphics[width=1.0\linewidth]{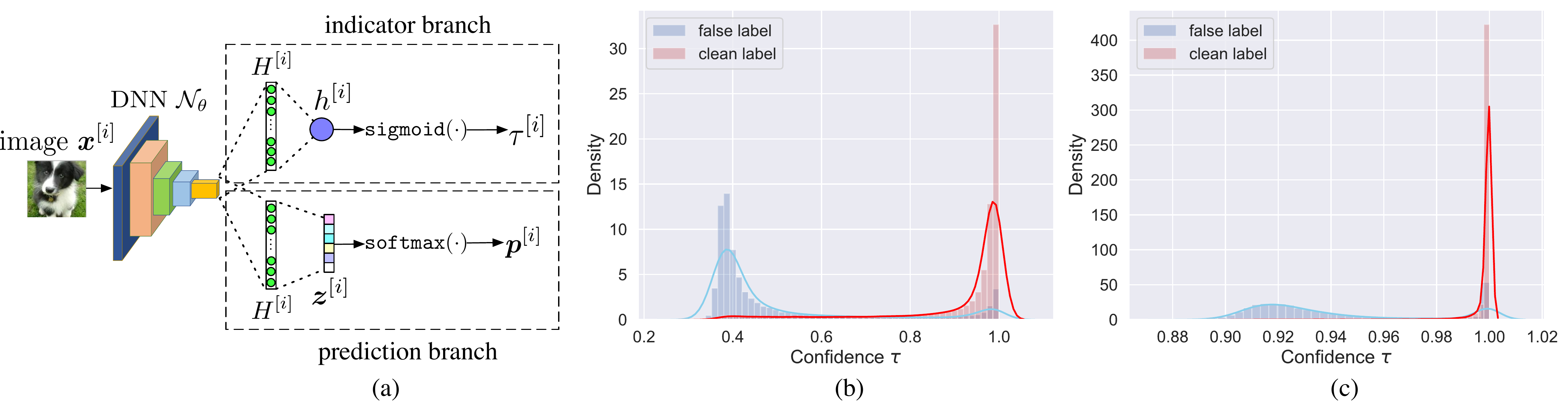}
	\end{center}
	\caption{In (a), we introduce an indicator branch in addition to the prediction branch. Given an input image $\bm{x}^{[i]}$, the indicator branch produces a single scalar value $\tau^{[i]}$ to indicate confidence and the prediction branch produces the softmax prediction probability $\bm{p}^{[i]}$. (b) and (c) show the density of confidence $\tau$ on the CIFAR-10 and CIFAR-100 with 40\% symmetric label noise respectively.}
	\label{fig:framework}
	\vspace{-1em}
\end{figure}

\vspace{-1em}
\subsection{Preliminary}
\label{subsec:pre}
Consider the $K$-class classification problem in noisy-label scenario, we have a training set $D=\{(\bm{x}^{[i]},\hat{y}^{[i]})\}^{N}_{i=1}$, where $\bm{x}^{[i]}$ is an input and $\hat{y}^{[i]} \in \mathcal{Y} = \{1,\dots,K\}$ is the corresponding noisy label. We denote $\bm{\hat{y}}^{[i]}\in \{0,1\}^{K}$ as one-hot vector of noisy label $\hat{y}^{[i]}$. The ground truth label $y$ is unavailable. A deep neural network model $\mathcal{N}_{\theta}$ (i.e. prediction branch in Figure \ref{fig:framework} (a)) maps an input $\bm{x}^{[i]}$ to a $K$-dimensional logits and then feeds the logits to a softmax function $\mathcal{S}(\cdot)$ to obtain $\bm{p}^{[i]}$ of the conditional probability of each class given $\bm{x}^{[i]}$, thus $\bm{p}^{[i]}=\mathcal{S}(\bm{z}^{[i]})$, $\bm{z}^{[i]}= \mathcal{N}_{\theta}(\bm{x}^{[i]})$.
$\theta$ denotes the parameters of the neural network and $\bm{z}^{[i]} \in \mathbb{R}^{K \times 1}$ denotes the $K$-dimensional logits (i.e. pre-softmax output). $\bm{z}^{[i]}$ is calculated by the fully connected layer from penultimate layer $H^{[i]} \in \mathbb{R}^{M\times 1}$. $\bm{z}^{[i]} = WH^{[i]}+\bm{b}$, where $W \in \mathbb{R}^{K\times M}$ denotes the weights and $\bm{b} \in \mathbb{R}^{K\times 1}$ denotes the bias in penultimate layer. Usually, the cross-entropy (CE) loss reflects how well the model fits the training set $D$:
\begin{align}
	\label{eq:ce}
	\mathcal{L}_{ce}=-\frac{1}{N}\sum_{i=1}^{N}(\bm{\hat{y}}^{[i]})^{T}\log(\bm{p}^{[i]}).
\end{align}
However, as noisy label $\hat{y}^{[i]}$ may be wrong with relatively high probability, the model gradually memorizes the samples with false labels when minimizing $\mathcal{L}_{ce}$ (in Figure \ref{fig:memorization} (a) and (b)).

\subsection{Confidence adaptive loss}
\label{sec:cal}
In addition to the prediction branch, we introduce an indicator branch just after the penultimate layer of the original model (in Figure \ref{fig:framework} (a)). The $M$-dimensional penultimate layer $H^{[i]}$ is shared in both branches. For each input $\bm{x}^{[i]}$, the prediction branch produces the softmax prediction $\bm{p}^{[i]}$ as usual. The indicator branch contains one or more fully connected layers to produce a single scalar value $h^{[i]}$, and $\mathtt{sigmoid}$ function is applied to scale it between 0 to 1. Assume we use one fully connected layer, $h^{[i]}=W'H^{[i]}+b'$, where $W' \in \mathbb{R}^{1\times M}$ denotes the weights and $b' \in \mathbb{R}$ denotes the bias in the penultimate layer of the indicator branch. Thus, we have
\begin{align}
	\tau^{[i]} = \mathtt{sigmoid} (h^{[i]}), \quad \tau^{[i]} \in (0,1), 
\end{align}
where $\tau^{[i]}$ denotes the confidence value of model prediction given input $\bm{x}^{[i]}$. The early-learning phenomenon reveals that the deep neural networks memorize the correctly-labeled samples before the mislabeled samples. Thus, we hypothesize that, a sample with a clean label \emph{in expectation} has a larger confidence value than a mislabeled sample in the early learning phase. To let confidence value $\tau$ capture it, we propose the confidence adaptive cross entropy (CACE) loss 
\begin{align}
	\mathcal{L}_{cace}=-\frac{1}{N}\sum_{i=1}^{N}(\bm{t}^{[i]})^{T}\log\big(\tau^{[i]}  
	(\bm{p}^{[i]}-\bm{t}^{[i]})+\bm{t}^{[i]}\big), 
\end{align}
where $\bm{t}^{[i]}$ is the \emph{target} vector for each sample $\bm{x}^{[i]}$. Generally, one can directly set $\bm{t}^{[i]}=\bm{\hat{y}}^{[i]}$. However, it is less effective as $\bm{\hat{y}}^{[i]}$ can be wrong, so we propose a strategy to estimate $\bm{t}^{[i]}$ in Section \ref{sec:target}.
Intuitively, $\mathcal{L}_{cace}$ can be explained in two-fold: 1) In the early-learning phase, the model does not overfit the mislabeled samples. Therefore, their $\bm{p}-\bm{t}$ remain large. By minimizing $\mathcal{L}_{cace}$, it forces $\tau$ of mislabeled samples toward 0 as desired. 2) As for correctly-labeled samples, the model memorizes them first, resulting in the small $\bm{p}-\bm{t}$. Thus, it makes $\tau$ have no influence on minimizing $\mathcal{L}_{cace}$ in the case of correctly-labeled samples. As a result, by only minimizing $\mathcal{L}_{cace}$, we may obtain a trivial optimization that the model always produces $\tau \rightarrow 0$ for any inputs. To avoid this lazy learning circumstance, we introduce a penalty loss $\mathcal{L}_{p}$ as a cost.
\begin{align}
	\mathcal{L}_{p}=-\frac{1}{N}\sum_{i=1}^{N}\log(\tau^{[i]}), 
\end{align}
wherein the target value of $\tau$ is always 1 for all inputs. By adding a term $\mathcal{L}_{p}$ to $\mathcal{L}_{cace}$, $\tau$ of correctly labeled samples are pushed to 1, and $\tau$ of mislabeled samples tend to 0 as expected. Hence, we define the confidence adaptive loss as
\begin{align}
\label{eq:Lcal}
 \mathcal{L}_{\textrm{CAL}}=\mathcal{L}_{cace}+\lambda \mathcal{L}_{p}\text{,} 
\end{align}
where $\lambda$ controls the strength of penalty loss. 
As we can see in Figure \ref{fig:framework} (b) and (c), the confidence value $\tau$ successfully segregates the mislabeled samples from correctly-labeled samples.


\subsection{Auxiliary regularization term}
\label{sec:car}
We observe that the early learning phenomenon is not obvious when a dataset contains too many classes (e.g. CIFAR100), i.e, the mean of $\tau$ distributions for clean samples and mislabeled samples are close to each other as shown in Figure \ref{fig:framework} (c). Then $\mathcal{L}_\text{CAL}$ is likely to be reduced to $\mathcal{L}_{ce}$. To enhance the performance in this situation, we need to make $\tau$ of mislabeled samples closer to 0. Hence we propose a reverse confidence adaptive cross entropy as an auxiliary regularization term.
\begin{align}
	\mathcal{L}_{r\textrm{-}cace}=-\frac{1}{N}\sum_{i=1}^{N}\big(\tau^{[i]}(\bm{p}^{[i]}-\bm{t}^{[i]})+\bm{t}^{[i]}\big)^{T}\log(\bm{t}^{[i]}). 
\end{align}
As the target $\bm{t}^{[i]}$ is inside of the logarithm in $\mathcal{L}_{r\textrm{-}cace}$, this could cause computational problem when $\bm{t}^{[i]}$ contains zeros. Similar to clipping operation, we solve it by defining $\log(0)=A$ (where $A$ is a negative constant), which will be proved important for the theoretical analysis in Section \ref{sec:theore}. Putting all together, the confidence adaptive regularization (CAR) is
\begin{align}
\label{eq:car}
	\mathcal{L}_{\textrm{CAR}} = \mathcal{L}_{\textrm{CAL}} + \beta \mathcal{L}_{r\textrm{-}cace} =  \mathcal{L}_{cace} + \lambda \mathcal{L}_{p} + \beta \mathcal{L}_{r\textrm{-}cace}, 
\end{align}
where $\beta$ controls the strength of regularization carried by $\mathcal{L}_{r\textrm{-}cace}$. In summary, $\mathcal{L}_{cace}$ is designed for learning confidence by exploiting the early-learning phenomenon. $\mathcal{L}_{p}$ is adopted for avoiding trivial solution. $\mathcal{L}_{r\textrm{-}cace}$ makes CAR robust to label noise even in challenging cases. 

\subsection{Target estimation}
\label{sec:target}
CAR requires a target probability $\bm{t}$ for each sample in the training set. To yield better performance, ELR \cite{liu2020early} and SELF \cite{nguyen2020self} use temporal ensembling \cite{laine2016temporal} solely based on model predictions to approximate the target $\bm{t}$. However, it may lose the information of the original training set, and the model predictions can be ambiguous in the early stage of training. Instead, we desire to correct the noisy labels and develop a strategy to estimate the target by utilizing the noisy label $\bm{\hat{y}}$, model prediction $\bm{p}$ and confidence value $\tau$. In each epoch, the target $\bm{t}^{[i]}$ of given $\bm{x}^{[i]}$ is updated by
\begin{align}
	\bm{t}^{[i]}=\left\{ \begin{array}{ll}
		\bm{\hat{y}}^{[i]} & \textrm{if } E < E_{c}\\
		\alpha \bm{t}^{[i]} + (1-\alpha) \bm{p}^{[i]} & \textrm{if $E \ge E_{c}$ and $\tau^{[i]} \ge \delta$}  \\
		\bm{t}^{[i]} & \textrm{otherwise},\\
	\end{array} \right. 
\end{align}
where $E$ is the current epoch number, $E_{c}$ is the epoch that starts performing target estimation and $0 \le \alpha < 1$ is the momentum. Performance is robust to the value of $E_{c}$. We fix the $E_{c}=60$ by default. Threshold $\delta$ is used to exclude ambiguous predictions with low confidence. Thus, our strategy enhances the stability of model predictions and gradually corrects the noisy labels. We evaluate the performance of CAR and CE with different target estimation strategies in Appendix \ref{apd:dif_stra}.


\subsection{Theoretical analysis}
\label{sec:theore}

\subsubsection{Gradient analysis}
\label{sec:grad_analysis}
For sample-wise analysis, we denote the true label of sample $\bm{x}$ as $y \in\{1,...,K\}$. The ground-truth distribution over labels for sample $\bm{x}$ is $q(y|\bm{x})$, and $\sum_{k=1}^{K}q(k|\bm{x})=1$. Consider the case of a single ground-truth label $y$, then $q(y|\bm{x})=1$ and $q(k|\bm{x})=0$ for all $k\ne y$. We denote the prediction probability as $p(k|\bm{x})$ and $\sum_{k=1}^{K}p(k|\bm{x})=1$. For notation simplicity, we denote $p_{k}$, $q_{k}$, $p_{y}$, $q_{y}$, $p_{j}$, $q_{j}$ as abbreviations for $p(k|\bm{x})$, $q(k|\bm{x})$, $p(y|\bm{x})$, $q(y|\bm{x})$, $p(j|\bm{x})$ and $q(j|\bm{x})$. Besides, we assume no target estimation is performed in the following analysis.

We first explain how the cross-entropy loss $\mathcal{L}_{ce}$ (Eq. (\ref{eq:ce})) fails in noisy-label scenario. The gradient of sample-wise cross entropy loss $\mathcal{L}_{ce}$ with respect to $z_{j}$ is 
\begin{align}
	\label{eq:Dce}
	\frac{\partial \mathcal{L}_{ce}}{\partial z_{j}}=\left\{ \begin{array}{ll}
		p_{j}-1 \le 0, & q_{j}=q_{y}=1\\\\
		p_{j} \ge 0, & q_{j}=0  \\
	\end{array} \right. 
\end{align}
In the noisy label scenario, if $j$ is the true class and equals $y$, but $q_{j}=0$ due to the label noise, the contribution of $\bm{x}$ to the gradient is reversed. The entry corresponding to the impostor class $j'$, is also reversed because $q_{j'}=1$, causing the gradient of mislabeled samples dominates (in Figure \ref{fig:gradient} (a) and (b)). Thus, performing stochastic gradient descent eventually results in memorization of the mislabeled samples.

\begin{lemma}
	\label{lamma1}
	For the loss function $\mathcal{L}_\text{CAL}$ given in Eq. (\ref{eq:Lcal}) and $\mathcal{L}_\text{CAR}$ in Eq. (\ref{eq:car}), the gradient of sample-wise $\mathcal{L}_\text{CAL}$ and $\mathcal{L}_\text{CAR}$ ($\beta=1$) with respect to the logits $z_{j}$ can be derived as
	\begin{subnumcases}{\label{eq:dcal} \frac{\partial \mathcal{L}_\textrm{CAL}}{\partial z_{j}}=}
		(p_{j}-1)\frac{p_{j}}{p_{j}-1+1/\tau} \le 0, & $q_{j}=q_{y}=1\  (\text{j is the true class for } \bm{x})$ \label{eq:dcal_isj} \\
		p_{j}\frac{p_{y}}{p_{y}-1+1/\tau}\ge 0, & $q_{j}=0 \  (\text{j is not the true class for }\bm{x})$  \label{eq:dcal_isnotj}
	\end{subnumcases}
	and
	\begin{subnumcases}{\label{eq:dcar} \frac{\partial \mathcal{L}_\textrm{CAR}}{\partial z_{j}}=}
		(p_{j}-1) \frac{ p_{j}}{p_{j}-1 +1/\tau}-A\tau p_{j}(p_{j}-1) \le 0, & $q_{j}=q_{y}=1$ \label{eq:dcar_isj}\\
		p_{j}\frac{p_{y}}{p_{y}-1+1/\tau}-A\tau p_{j}p_{y} \ge 0 , & $q_{j}=0$ \label{eq:dcar_isnotj}
	\end{subnumcases}
	respectively, where $A$ is a negative constant defined in Section \ref{sec:car}.
\end{lemma}

The proof of Lemma \ref{lamma1} is based on gradient derivation in two cases. We defer it in Appendix \ref{apd:proof}.


\begin{figure}[t]
	\begin{center}
		\includegraphics[width=1.0\linewidth]{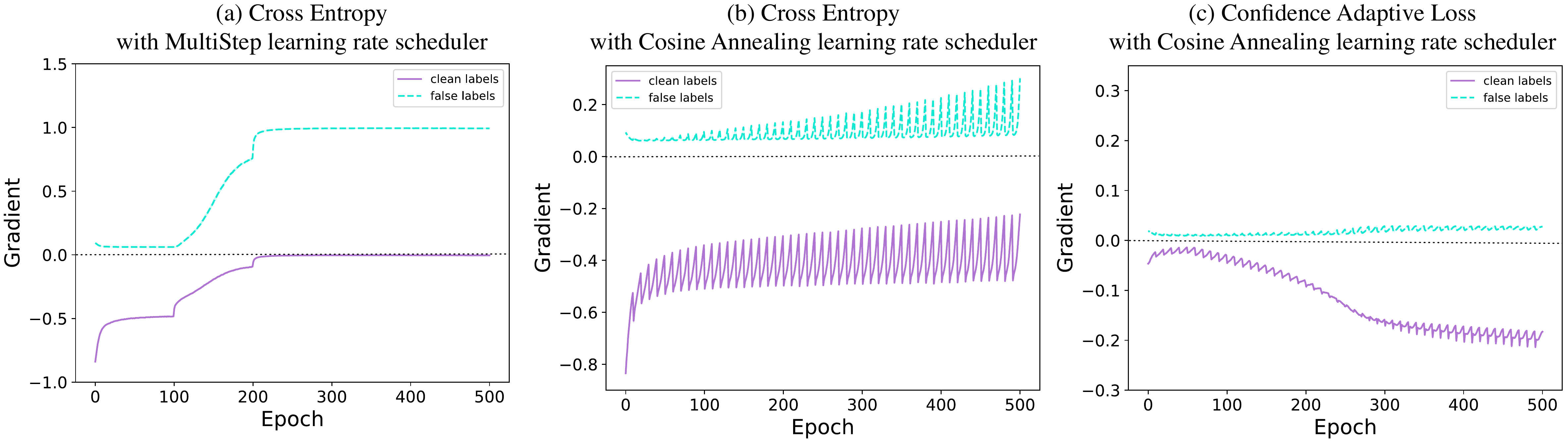}
	\end{center}
	\caption{On CIFAR-10 with 40\% symmetric label noise using ResNet34, we observe that in (a),  the gradient of clean labels dominates in early learning stage, but afterwards it vanishes and the gradient of false labels dominates. In (b), it only slows down this effect with cosine annealing learning rate scheduler. In (c), CAL effectively keeps the gradient of clean labels dominant and diminishes the gradient of false labels when epoch increases, preventing memorization of mislabeled samples.}
	\label{fig:gradient}
\end{figure}

\textbf{Gradient of $\mathcal{L}_\text{CAL}$ in Eq. (\ref{eq:dcal})}. Compared to the gradient of $\mathcal{L}_{ce}$ in Eq. (\ref{eq:Dce}), the gradient of $\mathcal{L}_\text{CAL}$ has an adaptive multiplier. We denote $Q=\frac{p_{j}}{p_{j}-1 +1/\tau}$. It is monotonically increasing on $\tau$ and $p_{j}$. We have $\lim_{\tau \rightarrow 1} Q = 1$, and $\lim_{\tau \rightarrow 0} Q = 0$. For the samples with the true class $j$ in Eq. (\ref{eq:dcal_isj}), the cross entropy gradient term $p_{j}-1$ of correctly-labeled samples tends to vanish after early learning stage because their $p_{j}$ is close to $q_{j}=1$, leading mislabeled samples to dominate the gradient. However, by multiplying $Q$ (note that $Q \rightarrow 0$ for mislabeled samples and $Q \rightarrow 1$ for correctly-labeled samples due to property of $\tau$ as we discussed in Section \ref{sec:cal}), it counteracts the effect of gradient dominating by mislabeled samples. For the samples that $j$ is not the true class in Eq. (\ref{eq:dcal_isnotj}), the gradient term $p_{j}$ is positive. Multiplying $Q<1$ effectively dampens the magnitudes of coefficients on these mislabeled samples, thereby diminishing their effect on the gradient (in Figure \ref{fig:gradient} (c)). 

\textbf{Gradient of $\mathcal{L}_\text{CAR}$ in Eq. (\ref{eq:dcar})}. Compared to the gradient of $\mathcal{L}_\text{CAL}$, an extra term derived from auxiliary regularization term $\mathcal{L}_{r\text{-}cace}$ is added. In the case of $q_{j}=q_{y}=1$ in Eq. (\ref{eq:dcar_isj}), the extra term $-A\tau p_{j}(p_{j}-1) < 0$ for $0 \le p_{j} \le 1$ and it is a convex quadratic function whose vertex is at $p_{j}=0.5$. It means the extra term $-A\tau p_{j}(p_{j}-1)$ provides the largest acceleration in learning around $p_{j}=0.5$ where the most ambiguous scenario occurs. Intuitively, the term $-A\tau p_{j}(p_{j}-1)$ pushes apart the peaks of $\tau$ distribution for correctly-labeled samples and mislabeled samples. In the case of $q_{j}=0$ in Eq. (\ref{eq:dcar_isnotj}), the extra term $-A\tau p_{j}p_{y} > 0$ is added. For correctly-labeled samples, $p_{y}$ is larger, adding $-A\tau p_{j}p_{y}$ leads the residual probabilities of other unlabeled classes reduce faster. For mislabeled samples, $p_{y}$ is close to 0, no acceleration needed. Overall, adding $\mathcal{L}_{r\text{-}cace}$ amplifies the effect of confidence learning in CAL, resulting in the confidence values of mislabeled samples become smaller. The empirical results of the influence of confidence distribution on CIFAR-100 with different strengths of $\mathcal{L}_{r\text{-}cace}$ are shown in Figure \ref{fig:effect_of_CAR}. 

\begin{figure}[t]
	\begin{center}
		\includegraphics[width=1.0\linewidth]{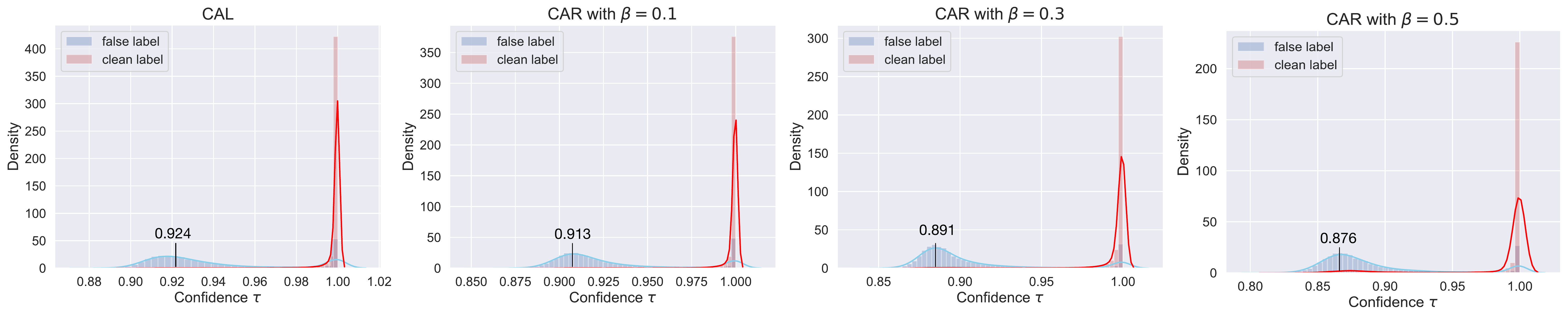}
	\end{center}
	\caption{The empirical density of confidence value $\tau$ on CIFAR-100 with 40\% symmetric label noise. The mean confidence values of mislabeled samples become smaller with the increasing of $\beta$.}
	\label{fig:effect_of_CAR}
	\vspace{-0.5em}
\end{figure}

\subsubsection{Label noise robustness}
Here we prove that the $\mathcal{L}_{r\textrm{-}cace}$ is robust to label noise following \cite{ghosh2017robust}. Recall that noisy label of $\bm{x}$ is $\hat{y} \in \{1,...,K\}$ and its true label is $y \in\{1,...,K\}$. We assume that the noisy sample $(\bm{x},\hat{y})$ is drawn from distribution $\mathcal{D}_{\eta}(\bm{x},\hat{y})$, and the ordinary sample $(\bm{x},y)$ is drawn from $\mathcal{D}(\bm{x},y)$. We have  $\hat{y} = i ( y = i)$ with probability $\eta_{ii}=(1-\eta)$ and $\hat{y}=j (y=i)$ with probability $\eta_{ij}$ for all $j\ne i$ and $\sum_{j\ne i}\eta_{ij}=\eta$. If $\eta_{ij}=\frac{\eta}{K-1}$ for all $j\ne i$, then the noise is \emph{uniform} or \emph{symmetric}, otherwise, the noise is \emph{class-conditional} or \emph{asymmetric}. Given any classifier $f$ and loss function $\mathcal{L}$, we define the risk of $f$ under clean labels as $\mathcal{R}_{\mathcal{L}}(f)=\mathbb{E}_{\mathcal{D}(\bm{x},y)}[\mathcal{L}(f(\bm{x},y))]$, and the risk under label noise rate $\eta$ as $\mathcal{R}^{\eta}_{\mathcal{L}}(f)=\mathbb{E}_{\mathcal{D}(\bm{x},\hat{y})}[\mathcal{L}(f(\bm{x},\hat{y}))]$. Let $f^{*}$ and $f^{*}_{\eta}$ be the global minimizers of $\mathcal{R}_{\mathcal{L}}(f)$ and $\mathcal{R}^{\eta}_{\mathcal{L}}(f)$ respectively. Then, the empirical risk minimization under loss function $\mathcal{L}$ is defined to be \emph{noise-tolerant} if $f^{*}$ is a global minimum of the noisy risk $\mathcal{R}^{\eta}_{\mathcal{L}}(f)$.

\begin{theorem}
	\label{theroem:1}
	Under symmetric or uniform label noise with noise rate $\eta < \frac{K-1}{K}$, we have  
	\begin{align}
		0\le \mathcal{R}_{\mathcal{L}_{r\textrm{-}cace}}(f^{*}_{\eta}) - \mathcal{R}_{\mathcal{L}_{r\textrm{-}cace}}(f^{*}) < \frac{-A\eta (K-1)}{K(1-\eta)-1} 
	\end{align}
	and 
	\begin{align}
			A\eta<\mathcal{R}^{\eta}_{\mathcal{L}_{r\textrm{-}cace}}(f^{*}_{\eta}) - \mathcal{R}^{\eta}_{\mathcal{L}_{r\textrm{-}cace}}(f^{*}) \le 0 
	\end{align}
	where $f^{*}$ and $f^{*}_{\eta}$ be the global minimizers of $\mathcal{R}_{\mathcal{L}_{r\textrm{-}cace}}(f)$ and $\mathcal{R}^{\eta}_{\mathcal{L}_{r\textrm{-}cace}}(f)$ respectively.
\end{theorem}

\begin{theorem}
	\label{theorem:2}
	Under class-dependent label noise with $\eta_{ij}<1-\eta_{i},\forall j\ne i,\forall i,j \in [K]$, where $\eta_{ij}=p(\hat{y}=j|y=i), \forall j\ne i$ and $(1-\eta_{i}) = p(\hat{y}=i|y=i)$, if $\mathcal{R}_{\mathcal{L}_{r\text{-}cace}}(f^{*})=0$, then
	\begin{align}
		0\le\mathcal{R}_{\mathcal{L}_{r\text{-}cace}}^{\eta}(f^{*}) - \mathcal{R}_{\mathcal{L}_{r\text{-}cace}}^{\eta}(f^{*}_{\eta}) < G, 
	\end{align}
	where $G=A(1-K) \mathbb{E}_{\mathcal{D}(\bm{x},y)}(1-\eta_{y}) >0 $, $f^{*}$ and $f^{*}_{\eta}$ be the global minimizers of $\mathcal{R}_{\mathcal{L}_{r\textrm{-}cace}}(f)$ and $\mathcal{R}^{\eta}_{\mathcal{L}_{r\textrm{-}cace}}(f)$ respectively. 
	
\end{theorem}

Due to the space constraints, we defer the proof of Theorem \ref{theroem:1} and Theorem \ref{theorem:2} to the Appendix \ref{apd:proof}. Theorem \ref{theroem:1} and Theorem \ref{theorem:2} ensure that by minimizing $\mathcal{L}_{r\textrm{-}cace}$ under symmetric and asymmetric label noise, the difference of the risks caused by the derived hypotheses $f^{*}_{\eta}$ and $f^{*}$ are always bounded. The bounds are related to the negative constant $A$. Since $A$ is the approximate of $\log(0)$ which is actually $-\infty$. A larger $A$ (closer to 0) leads to a tighter bound but introduces a larger approximation error in implementation. A reasonable $A$ we set is -4 in our experiments. We also compare $\mathcal{L}_{r\textrm{-}cace}$ with existing noise-robust loss functions in Appendix \ref{sec:sample_wise_loss}. 

\section{Experiments}

\textbf{Comparison with the state-of-the-art methods} We evaluate our approach on two benchmark datasets with simulated label noise, CIFAR-10 and CIFAR-100 \cite{krizhevsky2009learning}, and two real-world datasets, Clothing1M \cite{xiao2015learning} and WebVision \cite{li2017webvision}. More information of datasets, data preprocessing, label noise injection and training details can be found in Appendix \ref{sec:details_of_exp}. 

Table \ref{table:cifar10and100_resnet34} shows the performance of CAR on CIFAR-10 and CIFAR-100 with different levels of symmetric and asymmetric label noise. All methods use the same backbone (ResNet34). We compare CAR to the state-of-the-art approaches that only modify the training loss without extra regularization techniques, such as mixup data augmentation, two networks, and weight averaging. CAR obtains the highest performance in most cases and achieves comparable results in the most challenging cases (e.g. under 80\% symmetric noise).

Table \ref{table:clothing1m} compares CAR to state-of-the-art methods trained on the Clothing1M dataset. Note that DivideMix and ELR+ require mixup data augmentation, two networks, and weight averaging, while CAR is a pure regularization method. Except for DivideMix and ELR+, CAR slightly outperforms other methods.

\begin{table}
	\caption{Test Accuracy (\%) on CIFAR-10 and CIFAR-100 with various levels of label noise injected to the training set. We compare with previous works under the same backbone ResNet34. The results are averaged over 3 trials. Results are taken from their original papers. The best results are in \textbf{bold}. \emph{Note that SAT \cite{huang2020self}, ELR \cite{liu2020early} and CAR use cosine annealing learning rate scheduler \cite{loshchilov2016sgdr}.} } 
	\begin{center}
		\resizebox{1.0\textwidth}{!}{
			\centering
			\begin{tabular}{ p{25mm} c c c  c c cc c c c c } 
				\toprule
				\multicolumn{2}{c}{\multirow{1}{*}{Dataset}} & \multicolumn{5}{c}{CIFAR-10} &\multicolumn{5}{c}{CIFAR-100}\\ \cmidrule(lr){3-7} \cmidrule(lr){8-12}
				\multicolumn{2}{c}{\multirow{1}{*}{Noise type}} & \multicolumn{4}{c}{symm} &\multicolumn{1}{c}{asymm} & \multicolumn{4}{c}{symm}&\multicolumn{1}{c}{asymm} \\   \cmidrule(lr){3-6} \cmidrule(lr){7-7} \cmidrule(lr){8-11} \cmidrule(lr){12-12}
				\multicolumn{2}{c}{\multirow{1}{*}{Method/Noise ratio}}& 20\% & \multicolumn{1}{c}{40\%}&\multicolumn{1}{c}{60\%}&\multicolumn{1}{c}{80\%}& \multicolumn{1}{c}{40\%} & \multicolumn{1}{c}{20\%}&\multicolumn{1}{c}{40\%}&\multicolumn{1}{c}{60\%}&\multicolumn{1}{c}{80\%}&\multicolumn{1}{c}{40\%}\\
				\midrule
				\multirow{1}{*}{Cross Entropy} &  & 86.98 $\pm$ 0.12 &81.88 $\pm$ 0.29&74.14 $\pm$ 0.56&53.82 $\pm$ 1.04&80.11 $\pm$ 1.44&58.72 $\pm$ 0.26&48.20 $\pm$ 0.65&37.41 $\pm$ 0.94&18.10 $\pm$ 0.82&42.74 $\pm$ 0.61\\ 
				\multirow{1}{*}{Forward $\hat{T}$ \cite{patrini2017making}} && 87.99 $\pm$ 0.36 &83.25 $\pm$ 0.38&74.96 $\pm$ 0.65&54.64 $\pm$ 0.44&83.55 $\pm$ 0.58&39.19 $\pm$ 2.61&31.05 $\pm$ 1.44&19.12 $\pm$ 1.95&8.99 $\pm$ 0.58&34.44 $\pm$ 1.93\\ 
				
				\multirow{1}{*}{Bootstrap \cite{reed2014training}} && 86.23 $\pm$ 0.23 &82.23 $\pm$ 0.37&75.12 $\pm$ 0.56&54.12 $\pm$ 1.32&81.21 $\pm$ 1.47&58.27 $\pm$ 0.21&47.66 $\pm$ 0.55&34.68 $\pm$ 1.10&21.64 $\pm$ 0.97&45.12 $\pm$ 0.57\\ 
				
				\multirow{1}{*}{GCE \cite{zhang2018generalized}} & &89.83 $\pm$ 0.20 &87.13 $\pm$ 0.22&82.54 $\pm$ 0.23&64.07 $\pm$ 1.38&76.74 $\pm$ 0.61&66.81 $\pm$ 0.42&61.77 $\pm$ 0.24&53.16 $\pm$ 0.78&29.16 $\pm$ 0.74&47.22 $\pm$ 1.15\\ 
				\multirow{1}{*}{Joint Opt \cite{tanaka2018joint}} & & 92.25 &90.79&86.87&69.16&-&58.15&54.81&47.94&17.18&-\\
				\multirow{1}{*}{NLNL \cite{kim2019nlnl}} & & 94.23 &92.43&88.32&-&89.86&71.52&66.39&56.51&-&45.70\\
				\multirow{1}{*}{SL \cite{wang2019symmetric}} & & 89.83 $\pm$ 0.20 &87.13 $\pm$ 0.26&82.81 $\pm$ 0.61&68.12 $\pm$ 0.81&82.51 $\pm$ 0.45&70.38 $\pm$ 0.13&62.27 $\pm$ 0.22&54.82 $\pm$ 0.57&25.91 $\pm$ 0.44&69.32 $\pm$ 0.87\\
				\multirow{1}{*}{DAC \cite{thulasidasan2019combating}} & & 92.91 &90.71&86.30&74.84&-&73.55&66.92&57.17&32.16&-\\
				\multirow{1}{*}{SELF \cite{nguyen2020self}} & & - &91.13&-&63.59&-&-&66.71&-&35.56&-\\
				\multirow{1}{*}{SAT \cite{huang2020self}} & & 94.14&92.64&89.23&78.58&-&75.77&71.38&62.69&\textbf{38.72}&-\\
				
				\multirow{1}{*}{ELR \cite{liu2020early}} & & 92.12 $\pm$ 0.35&91.43 $\pm$ 0.21&88.87 $\pm$ 0.24&\textbf{80.69} $\pm$ \textbf{0.57}&90.35 $\pm$ 0.38&74.68 $\pm$ 0.31&68.43 $\pm$ 0.42&60.05 $\pm$ 0.78&30.27 $\pm$ 0.86&73.73 $\pm$ 0.34\\
				\multirow{1}{*}{CAR (Ours)}& &\textbf{94.37} $\pm$ \textbf{0.04}&\textbf{93.49} $\pm$ \textbf{0.07}&\textbf{90.56} $\pm$ \textbf{0.07} &\textbf{80.98} $\pm$ \textbf{0.27}& \textbf{92.09} $\pm$ \textbf{0.12}& \textbf{77.90} $\pm$ \textbf{0.14}&\textbf{75.38} $\pm$ \textbf{0.08}&\textbf{69.78} $\pm$ \textbf{0.69}&\textbf{38.24} $\pm$ \textbf{0.55}&\textbf{74.89} $\pm$ \textbf{0.20} \\ 
				\bottomrule

			\end{tabular}
		}
	\end{center}
	\label{table:cifar10and100_resnet34}
\end{table}

\begin{table*}
	\caption{Comparison with state-of-the-art methods trained on Clothing1M. Results of other methods are taken from original papers. All methods use an ResNet-50 architecture pretrained on ImageNet. }
	\begin{center}
		\resizebox{1.0\textwidth}{!}{
			\centering
			\begin{tabular}{  c c  c c c c c c c c c } 
				\toprule
				\multicolumn{1}{c}{CE}&\multicolumn{1}{c}{Forward \cite{patrini2017making}}&\multicolumn{1}{c}{GCE \cite{zhang2018generalized}} &\multicolumn{1}{c}{SL \cite{wang2019symmetric}}&\multicolumn{1}{c}{Joint-Optim \cite{tanaka2018joint}}&\multicolumn{1}{c}{DMI \cite{xu2019l_dmi}}&\multicolumn{1}{c}{ELR \cite{liu2020early}}&\multicolumn{1}{c}{ELR+ \cite{liu2020early}}&\multicolumn{1}{c}{DivideMix \cite{li2020dividemix}}&\multicolumn{1}{c}{CAR}\\
				\midrule
				 \multicolumn{1}{c}{69.21}&\multicolumn{1}{c}{69.84}&\multicolumn{1}{c}{69.75} &\multicolumn{1}{c}{71.02} & \multicolumn{1}{c}{72.16}&\multicolumn{1}{c}{72.46}&\multicolumn{1}{c}{72.87}&\multicolumn{1}{c}{\textbf{74.81}}&\multicolumn{1}{c}{74.76}&\multicolumn{1}{c}{73.19} \\
				\bottomrule

			\end{tabular}
		}
	\end{center}
	\label{table:clothing1m}
	\vspace{-0.8em}
\end{table*}

\begin{table*}[htb!]
	\caption{Comparison with state-of-the-art methods trained on mini WebVision. Results of other methods are taken from \cite{li2020dividemix,liu2020early}. All methods use an InceptionResNetV2 architecture. }
	\begin{center}
		\resizebox{1.0\textwidth}{!}{
			\centering
			\begin{tabular}{ p{25mm} c c c  c c c c c c c } 
				\toprule
				\multicolumn{3}{c}{\multirow{1}{*}{}} & \multicolumn{1}{c}{D2L \cite{ma2018dimensionality}}&\multicolumn{1}{c}{MentorNet \cite{jiang2017mentornet}}&\multicolumn{1}{c}{Co-teaching \cite{han2018co}} &\multicolumn{1}{c}{Iterative-CV \cite{chen2019understanding}}&\multicolumn{1}{c}{ELR \cite{liu2020early}}&\multicolumn{1}{c}{DivideMix \cite{li2020dividemix}}&\multicolumn{1}{c}{ELR+ \cite{liu2020early}}&\multicolumn{1}{c}{CAR}\\\midrule
				\multicolumn{2}{c}{\multirow{2}{*}{WebVision}} & \multicolumn{1}{c}{top1}&\multicolumn{1}{c}{62.68}&\multicolumn{1}{c}{63.00}&\multicolumn{1}{c}{63.58} &\multicolumn{1}{c}{65.24} & \multicolumn{1}{c}{76.26}&\multicolumn{1}{c}{77.32}&\multicolumn{1}{c}{\textbf{77.78}}&\multicolumn{1}{c}{77.41} \\ 
				\multicolumn{2}{c}{}& top5 & \multicolumn{1}{c}{84.00}&\multicolumn{1}{c}{81.40}&\multicolumn{1}{c}{85.20}& 85.34 & \multicolumn{1}{c}{91.26}&\multicolumn{1}{c}{91.64}&\multicolumn{1}{c}{91.68}&\multicolumn{1}{c}{\textbf{92.25}}\\
				\midrule
				\multicolumn{2}{c}{\multirow{2}{*}{ILSVRC12}} & \multicolumn{1}{c}{top1}&\multicolumn{1}{c}{57.80}&\multicolumn{1}{c}{57.80}&\multicolumn{1}{c}{61.48} &\multicolumn{1}{c}{61.60} & \multicolumn{1}{c}{68.71}&\multicolumn{1}{c}{\textbf{75.20}}&\multicolumn{1}{c}{70.29}&\multicolumn{1}{c}{74.09} \\ 
				\multicolumn{2}{c}{}& top5 & \multicolumn{1}{c}{81.36}&\multicolumn{1}{c}{79.92}&\multicolumn{1}{c}{84.70}& 84.98 & \multicolumn{1}{c}{87.84}&\multicolumn{1}{c}{90.84}&\multicolumn{1}{c}{89.76}&\multicolumn{1}{c}{\textbf{92.09}}\\
				\bottomrule

			\end{tabular}
		}
	\end{center}
	\label{table:webvision}
\end{table*}

Table \ref{table:webvision} compares CAR to state-of-the-art methods trained on the mini WebVision dataset and evaluated on both WebVision and ImageNet ILSVRC12 validation sets. On WebVision, CAR outperforms others on top5 accuracy, even better than DivideMix and ELR+. On top1 accuracy, CAR is slightly superior to DivideMix and achieves comparable performance to ELR+. On ILSVRC12, DivideMix achieves superior performance in terms of top1 accuracy, while CAR achieves the best top5 accuracy. We describe the hyperparameters sensitivity of CAR in Appendix \ref{sec:hyper_sensi}.

\textbf{Ablation study} Table \ref{tab:ablation} reports the influence of three individual components in CAR: auxiliary regularization term $\mathcal{L}_{r\textrm{-}cace}$, target estimation and indicator branch. Removing $\mathcal{L}_{r\textrm{-}cace}$ does not hurt the performance on CIFAR-10. However, the term $\mathcal{L}_{r\textrm{-}cace}$ improves the performance on CIFAR-100. The larger the noise is, the more improvement we obtain. Removing the target estimation leads to a significant performance drop. This suggests that estimating the target by properly using model predictions is crucial for avoiding memorization. To validate the effect of adding the indicator branch, we conduct another way to calculate confidence value without using indicator branch: using the highest probability as the confidence value, which means $\tau^{[i]} = \max_{j} \bm{p}^{[i]}_{j}, j\in[1,K]$. Without using the indicator branch, the model only converges in two easy cases. Hence, directly calculating the confidence by model output does interfere with the original prediction branch, while adding an extra indicator branch solves this problem.

\begin{table}
	\caption{Influence of three components in our approach. $\circleddash$ means the model fails to converge.} 
	\begin{center}
		\resizebox{0.8\textwidth}{!}{
			\centering
			\begin{tabular}{ p{20mm} c c  c c c c c } 
				\toprule
				\multicolumn{2}{c}{\multirow{1}{*}{Dataset}} & \multicolumn{3}{c}{CIFAR-10} &\multicolumn{3}{c}{CIFAR-100}\\ \cmidrule(lr){3-5} \cmidrule(lr){6-8}
				\multicolumn{2}{c}{\multirow{1}{*}{Noise type}} & \multicolumn{2}{c}{symm} &\multicolumn{1}{c}{asymm} & \multicolumn{2}{c}{symm}&\multicolumn{1}{c}{asymm} \\   \cmidrule(lr){3-4} \cmidrule(lr){5-5} \cmidrule(lr){6-7} \cmidrule(lr){8-8}
				\multicolumn{2}{c}{\multirow{1}{*}{Noise ratio}}&  \multicolumn{1}{c}{40\%}&\multicolumn{1}{c}{80\%}& \multicolumn{1}{c}{40\%} &\multicolumn{1}{c}{40\%}&\multicolumn{1}{c}{80\%}&\multicolumn{1}{c}{40\%}\\
				\midrule
				
				\multirow{1}{*}{CAR} & &\textbf{93.49} $\pm$ \textbf{0.07} &\textbf{80.98} $\pm$ \textbf{0.27}&\textbf{92.09} $\pm$ \textbf{0.12}&\textbf{75.38} $\pm$ \textbf{0.08}&\textbf{38.24} $\pm$ \textbf{0.55}&\textbf{74.89} $\pm$ \textbf{0.20}\\
				\multirow{1}{*}{-- $\mathcal{L}_{r\textrm{-}cace}$} & &93.49 $\pm$ 0.07 &80.98 $\pm$ 0.27&92.09 $\pm$ 0.12&74.65 $\pm$ 0.09&34.79 $\pm$ 0.71&74.73 $\pm$ 0.12\\
				\multirow{1}{*}{-- target estimation} & &89.47 $\pm$ 0.50&76.91 $\pm$ 0.22&88.23 $\pm$ 0.22&69.91 $\pm$ 0.21&31.33 $\pm$ 0.38&55.68 $\pm$ 0.17\\
				\multirow{1}{*}{-- indicator branch} & &90.94 $\pm$ 0.28&$\circleddash$&91.55 $\pm$ 0.07&$\circleddash$&$\circleddash$&$\circleddash$\\
				
				\bottomrule

			\end{tabular}
		}
	\end{center}
	\label{tab:ablation}
\end{table}

\begin{figure}[t]
	\centering
	\begin{minipage}[t]{0.45\textwidth}
		\centering
		\includegraphics[width=0.8\linewidth]{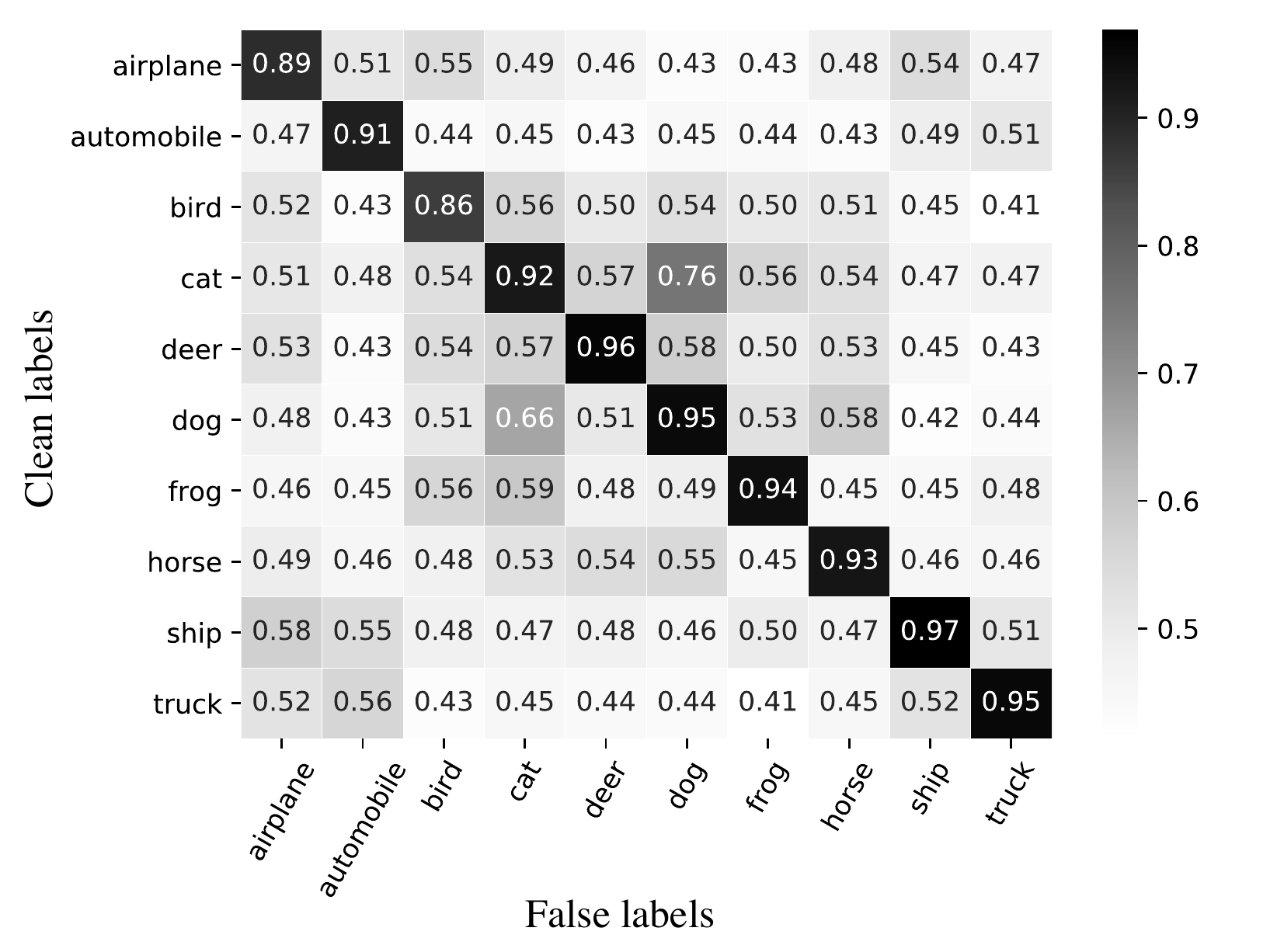}
		\caption{Average confidence values $\tau$ of false labels w.r.t clean labels on CIFAR-10 with 40\% symmetric label noise.}
		\label{fig:confidence_hm}
	\end{minipage}
	\hspace{0.5cm}
	\begin{minipage}[t]{0.45\textwidth}	
		\centering
		\includegraphics[width=0.8\linewidth]{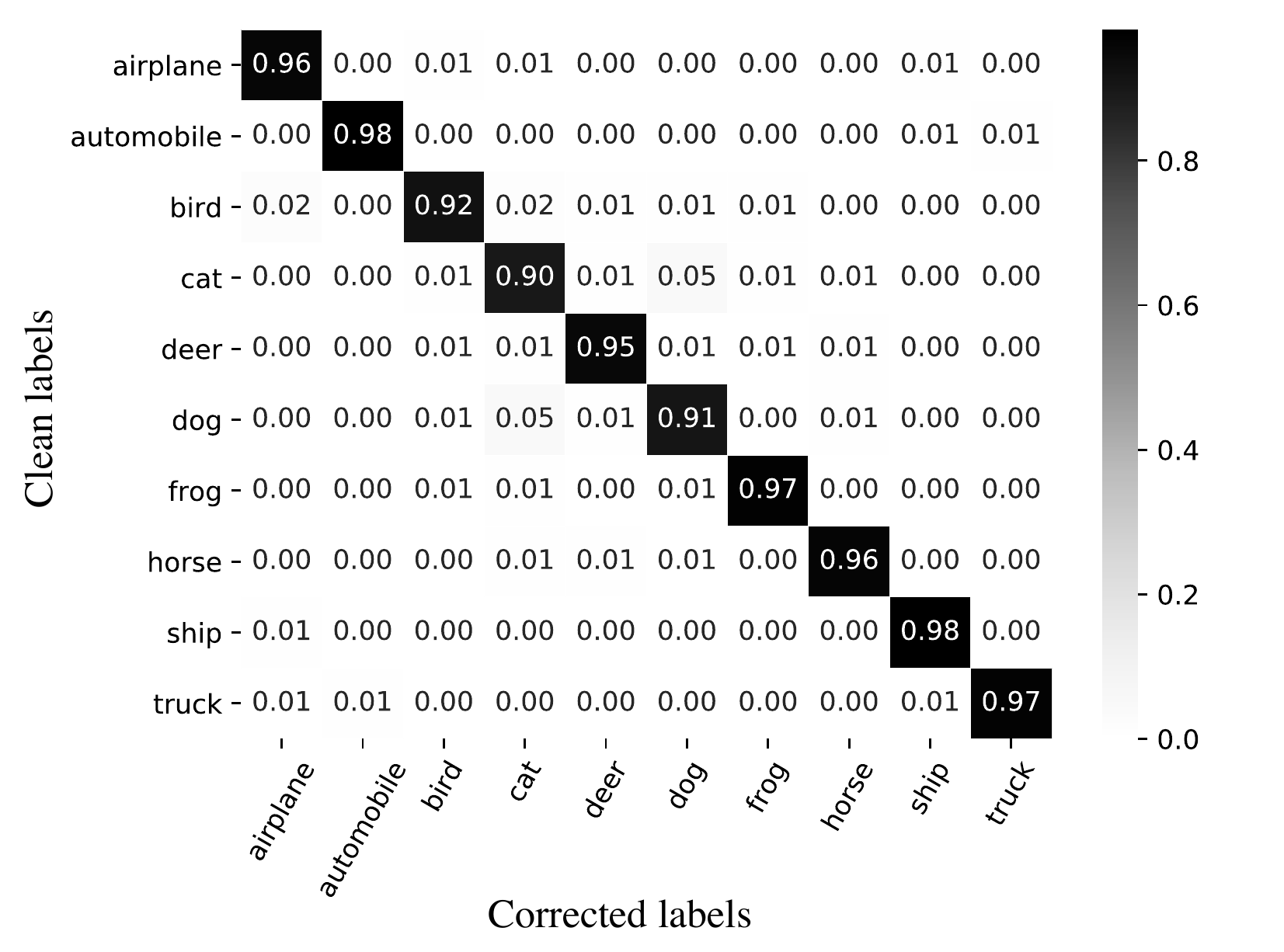}
		\caption{Confusion matrix of corrected labels w.r.t clean labels on CIFAR-10 with 40\% symmetric label noise.}
		\label{fig:label_recover}
	\end{minipage}
\vspace{-0.9em}
\end{figure}



\textbf{Identification of mislabeled samples} When exploiting the progress of the early learning phase by CAL, we have observed that the correctly-labeled samples have larger confidence values than the mislabeled samples. We report the average confidence values of samples in Figure \ref{fig:confidence_hm}. The $(i,j)$-th block represents the average confidence value of samples with clean label $i$ and false label $j$. We observe that the confidence values on the diagonal blocks are higher than those on non-diagonal blocks, which means that the confidence value has an effect similar to the probability of extra class in DAC \cite{thulasidasan2019combating} and AUM \cite{pleiss2020identifying}. The key difference is that DAC and AUM perform two stages of training: identify the mislabeled samples and then drop them to perform classification, while we incorporate the confidence values in loss function and implicitly achieve the regularization effect to avoid memorization of mislabeled samples. 

\textbf{Label correction} Recall that we perform target estimation in Section \ref{sec:target}. Since the target is calculated by a moving average between noisy labels and model predictions, our approach is able to gradually correct the false labels. The correction accuracy can be calculated by $\frac{1}{N}\sum_{i}^{N}\mathbbm{1} \{\mathrm{argmax}\ \bm{y}^{[i]}=\mathrm{argmax}\ \bm{t}^{[i]}\}$, where $\bm{y}^{[i]}$ is the clean label of training sample $\bm{x}^{[i]}$. We evaluate the correction accuracy on CIFAR-10 and CIFAR-100 with 40\% symmetric label noise. CAR obtains correction accuracy of 95.1\% and 86.4\%, respectively. The confusion matrix of corrected labels w.r.t the clean labels on CIFAR-10 is shown in Figure \ref{fig:label_recover}. As we can see, CAR corrects the false labels impressively well for all classes. More results on real-world datasets can be found in Appendix \ref{sec:more_label_correct}.



\section{Conclusion}

Based on the early learning and memorization phenomenon of deep neural networks in the presence of noisy labels, we propose an adaptive regularization method that implicitly adjusts the gradient to prevent memorization on noisy labels. Through extensive experiments across multiple datasets, our approach yields comparable or even superior results to the state-of-the-art methods. 


\bibliographystyle{unsrtnat}
\bibliography{egbib}
\newpage



\appendix
\section{Theoretical analysis}
\setcounter{equation}{0}
\setcounter{lemma}{0}
\setcounter{theorem}{0}

\subsection{Gradient derivation of $\mathcal{L}_{\textrm{CAL}}$ and $\mathcal{L}_{\textrm{CAR}}$}
\label{apd:Grad_derivation}
Assume the target $t$ equals to ground truth distribution. The sample-wise $\mathcal{L}_{\textrm{CAL}}$ can be rewrite as:
\begin{align}
	\mathcal{L}_{\textrm{CAL}}=\mathcal{L}_{cace}+\lambda\mathcal{L}_{p} =-\sum_{k=1}^{K}q_{k}\log(\tau (p_{k}-q_{k})+q_{k})-\lambda \log \tau.
\end{align}
The derivation of the $\mathcal{L}_{\textrm{CAL}}$ with respect to the logits is as follows:
\begin{align}
	\label{deri_cal}
	\frac{\partial \mathcal{L}_{\textrm{CAL}}}{\partial z_{j}}=\frac{\partial \mathcal{L}_{cace}}{\partial z_{j}}=-\sum_{k=1}^{K}\frac{\tau q_{k}}{\tau (p_{k}-q_{k})+q_{k}}\frac{\partial p_{k}}{\partial z_{j}}.
\end{align}
Since $p_{k}=\mathcal{S}(\bm{z})$=$\frac{e^{z_{k}}}{\sum_{j=1}^{K}e^{z_{j}}}$, we have 
\begin{align}
	\frac{\partial p_{k}}{\partial z_{j}}=\frac{\partial \big(\frac{e^{z_{k}}}{\sum_{j=1}^{K}e^{z_{j}}} \big)}{\partial z_{j}}=\frac{\frac{\partial e^{z_{k}}}{\partial z_{j}}(\sum_{j=1}^{K}e^{z_{j}})-e^{z_{k}}\frac{\partial\big( \sum_{j=1}^{K}e^{z_{j}} \big)}{\partial z_{j}}}{(\sum_{j=1}^{K}e^{z_{j}})^2}.
\end{align}
In the case of $k=j:$
\begin{align}
	\label{deri_pk1}
	\frac{\partial p_{k}}{\partial z_{j}}=&\frac{\frac{\partial e^{z_{k}}}{\partial z_{k}}(\sum_{k=1}^{K}e^{z_{k}})-e^{z_{k}}\frac{\partial\big( \sum_{k=1}^{K}e^{z_{k}} \big)}{\partial z_{k}}}{(\sum_{k=1}^{K}e^{z_{k}})^2}=\frac{e^{z_{k}}(\sum_{k=1}^{K}e^{z_{k}})-e^{z_{k}}\cdot e^{z_{k}}}{(\sum_{k=1}^{K}e^{z_{k}})^2} \nonumber \\
	=&\frac{e^{z_{k}}}{\sum_{k=1}^{K}e^{z_{k}}}-\Big(\frac{e^{z_{k}}}{\sum_{k=1}^{K}e^{z_{k}}}\Big)^2=p_{k}-p_{k}^2.
\end{align}
In the case of $k\ne j:$
\begin{align}
	\label{deri_pk2}
	\frac{\partial p_{k}}{\partial z_{j}}=&\frac{0\cdot(\sum_{j=1}^{K}e^{z_{j}})-e^{z_{k}}\cdot e^{z_{j}}}{(\sum_{j=1}^{K}e^{z_{j}})^2}=-\frac{e^{z_{k}}}{\sum_{j=1}^{K}e^{z_{j}}} \frac{e^{z_{j}}}{\sum_{j=1}^{K}e^{z_{j}}} = -p_{k}p_{j}.
\end{align}
Combining Eq. (\ref{deri_pk1}) and (\ref{deri_pk2}) into Eq. (\ref{deri_cal}), we obtain:
\begin{align}
	\frac{\partial \mathcal{L}_{\textrm{CAL}}}{\partial z_{j}}=&-\sum_{k=1}^{K}\frac{\tau q_{k}}{\tau (p_{k}-q_{k})+q_{k}}\frac{\partial p_{k}}{\partial z_{j}}\nonumber\\=&-\frac{\tau q_{j}}{\tau (p_{j}-q_{j})+q_{j}}\frac{\partial p_{j}}{\partial z_{j}}-\sum_{k\ne j}^{K}\frac{\tau q_{k}}{\tau (p_{k}-q_{k})+q_{k}}\frac{\partial p_{k}}{\partial z_{j}} \nonumber \\ =&-\frac{\tau q_{j}}{\tau (p_{j}-q_{j})+q_{j}}(p_{j}-p_{j}^2)-\sum_{k\ne j}^{K}\frac{\tau q_{k}}{\tau (p_{k}-q_{k})+q_{k}}(-p_{k}p_{j}) \nonumber \\=&-\frac{\tau q_{j}p_{j}}{\tau(p_{j}-q_{j})+q_{j}}+p_{j}\sum_{k=1}^{K}\frac{\tau q_{k}p_{k}}{\tau(p_{k}-q_{k})+q_{k}}.
\end{align}
Therefore, if $q_{j}=q_{y}=1$, then
\begin{align}
	\frac{\partial \mathcal{L}_{\textrm{CAL}}}{\partial z_{j}}=&-\frac{\tau p_{j}}{\tau p_{j} - \tau +1}+p_{j}\frac{\tau q_{j}p_{j}}{\tau(p_{j}-1)+1} =(p_{j}-1) \frac{\tau p_{j}}{\tau p_{j}-\tau +1}= (p_{j}-1) \frac{ p_{j}}{p_{j}-1 +1/\tau}. 
\end{align}
If $q_{j}=0$, then 
\begin{align}
	\frac{\partial \mathcal{L}_{\textrm{CAL}}}{\partial z_{j}}=&p_{j}\frac{\tau q_{y}p_{y}}{\tau (p_{y}-q_{y})+q_{y}} = p_{j} \frac{p_{y}}{p_{y}-1+1/\tau}. 
\end{align}
The sample-wise $\mathcal{L}_\textrm{CAR}$ can be rewrite as (assume $\beta=1$):
\begin{align}
	\mathcal{L}_\textrm{CAR}=\mathcal{L}_\textrm{CAL} + \beta\mathcal{L}_{r\textrm{-}cace} = \mathcal{L}_\textrm{CAL} - \sum_{k=1}^{K} (\tau (p_{k}-q_{k})+q_{k}) \log q_{k}. 
\end{align}
Since we have obtain the gradient of $\mathcal{L}_\textrm{CAL}$, we now only analyze the gradient of $\mathcal{L}_{r\textrm{-}cace}$ with respect to the logits as follows:
\begin{align}
	\label{deri_rcace}
	\frac{\partial \mathcal{L}_{r\textrm{-}cace}}{\partial z_{j}}=-\sum_{k=1}^{K}\frac{\tau \partial p_{k} }{\partial z_{j}}\log q_{k}.
\end{align}
Combining Eq. (\ref{deri_pk1}) and (\ref{deri_pk2}), into Eq. (\ref{deri_rcace}), we have 
\begin{align}
	\frac{\partial \mathcal{L}_{r\textrm{-}cace}}{\partial z_{j}}=&-\tau (p_{j}-p_{j}^{2})\log q_{j} -\tau\sum_{k\ne j}^{K}(-p_{k}p_{j})\log q_{k} \nonumber \\
	=&-\tau p_{j}\log q_{j}+\tau\sum_{k=1}^{K}p_{k}p_{j}\log q_{k}.
\end{align}
We denote $\log0 = A$, thus if $q_{j}=q_{y}=1$, then
\begin{align}
	\frac{\partial \mathcal{L}_{r\textrm{-}cace}}{\partial z_{j}}=-\tau p_{j}\log1+\tau p_{j}(p_{j}\log1+\sum_{k\ne j}^{K}p_{k}\log0)=\tau p_{j}(1-p_{j})A=-A\tau p_{j}(p_{j}-1).
\end{align}
If $q_{j}=0$, then
\begin{align}
	\frac{\partial \mathcal{L}_{r\textrm{-}cace}}{\partial z_{j}}=-\tau p_{j}\log0+\tau p_{j}(p_{y}\log1 + (1-p_{y})\log0 ) = -A\tau p_{j}+\tau p_{j}(1-p_{y})A=-A\tau p_{j}p_{y}. 
\end{align}
Therefore, the gradients of $\mathcal{L}_\textrm{CAR}$ is
\begin{align}
	\frac{\partial \mathcal{L}_\textrm{CAR}}{\partial z_{j}}=\left\{ \begin{array}{ll}
		(p_{j}-1) \frac{ p_{j}}{p_{j}-1 +1/\tau}-A\tau p_{j}(p_{j}-1) , & q_{j}=q_{y}=1\\\\
		p_{j}\frac{p_{y}}{p_{y}-1+1/\tau}-A\tau p_{j}p_{y} , & q_{j}=0  \\
	\end{array} \right.
\end{align}

\subsection{Formal proof for Lemma 1, Lemma2, Theorem 1 and Theorem 2}
\label{apd:proof}
\begin{lemma}
	\label{lemma1}
	For the loss function $\mathcal{L}_\text{CAL}$ given in Eq. (5) and $\mathcal{L}_\text{CAR}$ in Eq. (7), the gradient of sample-wise $\mathcal{L}_\text{CAL}$ and $\mathcal{L}_\text{CAR}$ ($\beta=1$) with respect to the logits $z_{j}$ can be derived as
	\begin{align}
		\frac{\partial \mathcal{L}_\textrm{CAL}}{\partial z_{j}}=\left\{ \begin{array}{lll}
			(p_{j}-1)\frac{p_{j}}{p_{j}-1+1/\tau} \le 0, & q_{j}=q_{y}=1\quad& (\text{j is the true class for sample } \bm{x})\\\\
			p_{j}\frac{p_{y}}{p_{y}-1+1/\tau}\ge 0, & q_{j}=0 \quad &(\text{j is not the true class for sample }\bm{x})  \nonumber\\
		\end{array} \right.
	\end{align}
	and
	\begin{align}
		\frac{\partial \mathcal{L}_\textrm{CAR}}{\partial z_{j}}=\left\{ \begin{array}{ll}
			(p_{j}-1) \frac{ p_{j}}{p_{j}-1 +1/\tau}-A\tau p_{j}(p_{j}-1) \le 0, & q_{j}=q_{y}=1\\\\
			p_{j}\frac{p_{y}}{p_{y}-1+1/\tau}-A\tau p_{j}p_{y} \ge 0 , & q_{j}=0 \nonumber \\
		\end{array} \right.
	\end{align}
	respectively.
\end{lemma}
\begin{proof} From the Appendix \ref{apd:Grad_derivation}, we obtain the gradient of the sample-wise $\mathcal{L}_\text{CAL}$ with respect to the logits $z_{j}$ is
	\begin{align}
		\label{eq:Dcal_apd}
		\frac{\partial \mathcal{L}_\text{CAL}}{\partial z_{j}}=\frac{\partial \mathcal{L}_{cace}}{\partial z_{j}}=-\sum_{k=1}^{K}\frac{\tau q_{k}}{\tau (p_{k}-q_{k})+q_{k}}\frac{\partial p_{k}}{\partial z_{j}}
	\end{align}
	where $\frac{\partial p_{k}}{\partial z_{j}}$ can be further derived base on whether $k=j$ by follows:
	\begin{align}
		\label{eq:Dpk_apd}
		\frac{\partial p_{k}}{\partial z_{j}}=\left\{ \begin{array}{ll}
			p_{k}-p_{k}^{2} & k=j\\
			-p_{j}p_{k} & k\ne j  \\
		\end{array} \right.
	\end{align}
	According to Eq. (\ref{eq:Dcal_apd}) and (\ref{eq:Dpk_apd}), the gradient of $\mathcal{L}_\textrm{CAL}$ can be derived as:
	\begin{align}
		\frac{\partial \mathcal{L}_\textrm{CAL}}{\partial z_{j}}=\left\{ \begin{array}{ll}
			(p_{j}-1)\frac{p_{j}}{p_{j}-1+1/\tau}, & q_{j}=q_{y}=1\\\\
			p_{j}\frac{p_{y}}{p_{y}-1+1/\tau}, & q_{j}=0  \\
		\end{array} \right.
	\end{align}
	Since $p_{j} \le 1$, we have $p_{j}-1 \le 0$. As $\tau < 1$, the term $\frac{p_{y}}{p_{y}-1 +1/\tau} >0 $, we have $(p_{j}-1)\frac{p_{y}}{p_{y}-1 +1/\tau} \le 0$ and $p_{j}\frac{p_{y}}{p_{y}-1 +1/\tau} \ge 0$. Similarly, the gradient of simplified $\mathcal{L}_\text{CAR}$ ($\beta=1$) can be derived as:
	\begin{align}
		\label{eq:Dcar_apd}
		\frac{\partial \mathcal{L}_\textrm{CAR}}{\partial z_{j}}=\frac{\partial \mathcal{L}_\textrm{CAL}}{\partial z_{j}}+\frac{\partial \mathcal{L}_{r\text{-}cace}}{\partial z_{j}}=\left\{ \begin{array}{ll}
			(p_{j}-1) \frac{ p_{j}}{p_{j}-1 +1/\tau}-A\tau p_{j}(p_{j}-1), & q_{j}=q_{y}=1\\\\
			p_{j}\frac{p_{y}}{p_{y}-1+1/\tau}-A\tau p_{j}p_{y} , & q_{j}=0  \\
		\end{array} \right.
	\end{align}
	Since $A$ is a negative constant, we obtain $-A\tau p_{j}(p_{j}-1) \le 0$. Thus, in the case of $q_{j}=q_{y}=1$, $\frac{\partial \mathcal{L}_\textrm{CAR}}{\partial z_{j}} \le 0$ and in the case of $q_{j}=0$, $\frac{\partial \mathcal{L}_\textrm{CAR}}{\partial z_{j}} \ge 0$ as claimed. Complete derivations can be found in the Appendix \ref{apd:Grad_derivation}.  
\end{proof}

The result in Lemma \ref{lemma1} ensures that, during the gradient decent, learning continues on true classes when trained with $\mathcal{L}_\textrm{CAL}$ and $\mathcal{L}_\textrm{CAR}$. We then prove the noise robustness of $\mathcal{L}_{r\text{-}cace}$.

Recall that noisy label of $\bm{x}$ is $\hat{y} \in \{1,...,K\}$ and its true label is $y \in\{1,...,K\}$. We assume that the noisy sample $(\bm{x},\hat{y})$ is drawn from distribution $\mathcal{D}_{\eta}(\bm{x},\hat{y})$, and the ordinary sample $(\bm{x},y)$ is drawn from $\mathcal{D}(\bm{x},y)$. Note that this paper follows the most common setting where label noise is \emph{instance-independent}. Then we have $\hat{y} = i ( y = i)$ with probability $\eta_{ii}=(1-\eta)$ and $\hat{y}=j (y=i)$ with probability $\eta_{ij}$ for all $j\ne i$ and $\sum_{j\ne i}\eta_{ij}=\eta$. If $\eta_{ij}=\frac{\eta}{K-1}$ for all $j\ne i$, then the noise is said to be \emph{uniform} or \emph{symmetric}, otherwise, the noise is said to be \emph{class-conditional} or \emph{asymmetric}. Given any classifier $f$ and loss function $\mathcal{L}$, we define the risk of $f$ under clean labels as $\mathcal{R}_{\mathcal{L}}(f)=\mathbb{E}_{\mathcal{D}(\bm{x},y)}[\mathcal{L}(f(\bm{x},y))]$, and the risk under label noise rate $\eta$ as $\mathcal{R}^{\eta}_{\mathcal{L}}(f)=\mathbb{E}_{\mathcal{D}(\bm{x},\hat{y})}[\mathcal{L}(f(\bm{x},\hat{y}))]$. Let $f^{*}$ and $f^{*}_{\eta}$ be the global minimizers of $\mathcal{R}_{\mathcal{L}}(f)$ and $\mathcal{R}^{\eta}_{\mathcal{L}}(f)$ respectively. Then, the empirical risk minimization under loss function $\mathcal{L}$ is defined to be \emph{noise-tolerant} if $f^{*}$ is a global minimum of the noisy risk $\mathcal{R}^{\eta}_{\mathcal{L}}(f)$.
\begin{lemma}
	\label{lamma2}
	For any $\bm{x}$, the sum of $\mathcal{L}_{r\text{-}cace}$ with respect to all the classes satisfies:
	\begin{align}
		0 < \sum_{j=1}^{K}\mathcal{L}_{r\text{-}cace}(f(\bm{x}),j) < A(1-K),
	\end{align}
	where $A=\log(0)$ is a negative constant that depends on the clipping operation.
\end{lemma}
\begin{proof}
	By the definition of $\mathcal{L}_{r\textrm{-}cace}$, we can rewrite the sample-wise $\mathcal{L}_{r\textrm{-}cace}$ as 
	\begin{align}
		\mathcal{L}_{r\textrm{-}cace}&=-\sum_{k=1}^{K}\big(\tau (p(k|\bm{x})-q(k|\bm{x}))+q(k|\bm{x})\big)\log q(k|\bm{x}) \nonumber \\
		&= - \big(\tau (p(y|\bm{x})-  q(y|\bm{x})) + q(y|\bm{x})\big)\log q(y|\bm{x}) - {\sum}_{k\ne y}\big(\tau (p(k|\bm{x})-q(k|\bm{x}))+q(k|\bm{x})\big)\log q(k|\bm{x}) \nonumber \\ 
		&= - \big(\tau p(y|\bm{x})-\tau + 1\big)\log 1 - A\tau {\sum}_{k\ne y}p(k|\bm{x}) \nonumber\\ 
		&= - A\tau(1-p(y|\bm{x})).
	\end{align}
	Therefore, we have
	\begin{align}
		\sum_{j=1}^{K}\mathcal{L}_{r\textrm{-}cace}(f(\bm{x}),j) =\sum_{j=1}^{K} - A\tau(1-p(j|\bm{x})) =-A\tau K + A\tau\sum_{j=1}^{K} p(j|\bm{x}) &= A\tau(1-K) \nonumber
	\end{align}
	As $\tau \in (0,1)$, $A$ is a negative constant, $K$ is a constant, hence 
	\begin{align}
		0 < \sum_{j=1}^{K}\mathcal{L}_{r\text{-}cace}(f(\bm{x}),j) < A(1-K), \nonumber
	\end{align}
	which concludes the proof.
\end{proof}

\begin{theorem}
	Under symmetric or uniform label noise with noise rate $\eta < \frac{K-1}{K}$, we have  
	\begin{align}
		0\le \mathcal{R}_{\mathcal{L}_{r\textrm{-}cace}}(f^{*}_{\eta}) - \mathcal{R}_{\mathcal{L}_{r\textrm{-}cace}}(f^{*}) \nonumber < \frac{-A\eta (K-1)}{K(1-\eta)-1} \nonumber
	\end{align}
	and 
	\begin{align}
		A\eta<\mathcal{R}^{\eta}_{\mathcal{L}_{r\textrm{-}cace}}(f^{*}_{\eta}) - \mathcal{R}^{\eta}_{\mathcal{L}_{r\textrm{-}cace}}(f^{*}) \le 0 \nonumber
	\end{align}
	where $f^{*}$ and $f^{*}_{\eta}$ be the global minimizers of $\mathcal{R}_{\mathcal{L}_{r\textrm{-}cace}}(f)$ and $\mathcal{R}^{\eta}_{\mathcal{L}_{r\textrm{-}cace}}(f)$ respectively.
\end{theorem}
\begin{proof}
	For symmetric noise, we have, for any $f$ \footnote{In the following, note that $\mathbb{E}_{\bm{x}}\mathbb{E}_{y|\bm{x}}=\mathbb{E}_{\bm{x},y}=\mathbb{E}_{\mathcal{D}(\bm{x},y)}$, which denote expectation with respect to the corresponding conditional distributions.}
	\begin{align}
		\mathcal{R}^{\eta}_{\mathcal{L}_{r\textrm{-}cace}}(f) &= \nonumber \mathbb{E}_{\mathcal{D}_{\eta}(\bm{x},\hat{y})}[\mathcal{L}_{r\textrm{-}cace}(f(\bm{x}),\hat{y})] \\
		&=\mathbb{E}_{\bm{x}}\mathbb{E}_{\mathcal{D}(y|\bm{x})}\mathbb{E}_{\mathcal{D}(\hat{y}|\bm{x},y)}[\mathcal{L}_{r\textrm{-}cace}(f(\bm{x}),\hat{y})]\nonumber\\\nonumber
		&=\mathbb{E}_{\mathcal{D}(\bm{x},y)}\big[(1-\eta)\mathcal{L}_{r\textrm{-}cace}(f(\bm{x}),y) +\frac{\eta}{K-1}{\sum}_{j\ne y}\mathcal{L}_{r\textrm{-}cace}(f(\bm{x}),j)\big] \nonumber \\
		&=(1-\eta)\mathcal{R}_{\mathcal{L}_{r\textrm{-}sace}}(f)+\frac{\eta}{K-1}\Big(\sum_{j=1}^{K}\mathcal{L}_{r\textrm{-}cace}(f(\bm{x}),j)-\mathcal{R}_{\mathcal{L}_{r\textrm{-}cace}}(f)\Big) \nonumber \\
		&= (1-\frac{\eta K}{K-1})\mathcal{R}_{\mathcal{L}_{r\textrm{-}cace}}(f) + \frac{\eta}{K-1}\sum_{j=1}^{K}\mathcal{L}_{r\textrm{-}cace}(f(\bm{x}),j) \nonumber
	\end{align}
	From Lemma \ref{lamma2}, for all $f$, we have:
	\begin{align}
		\psi\mathcal{R}_{\mathcal{L}_{r\textrm{-}cace}}(f)	 < \mathcal{R}^{\eta}_{\mathcal{L}_{r\textrm{-}cace}}(f) < -A\eta +\psi\mathcal{R}_{\mathcal{L}_{r\textrm{-}cace}}(f) \nonumber
	\end{align}
	where $\psi=(1-\frac{\eta K}{K-1})$. Since $\eta < \frac{K-1}{K}$, we have $\psi >0 $. Thus, we can rewrite the inequality in terms of $\mathcal{R}_{\mathcal{L}_{r\textrm{-}cace}}(f)$:
	\begin{align}
		\frac{1}{\psi}(\mathcal{R}^{\eta}_{\mathcal{L}_{r\textrm{-}cace}}(f) +A\eta)
		<\mathcal{R}_{\mathcal{L}_{r\textrm{-}cace}}(f) <\frac{1}{\psi} \mathcal{R}^{\eta}_{\mathcal{L}_{r\textrm{-}cace}}(f)\nonumber
	\end{align}
	Thus, for $f^{*}_{\eta}$, 
	\begin{align}
		\mathcal{R}_{\mathcal{L}_{r\textrm{-}cace}}(f^{*}_{\eta}) - \mathcal{R}_{\mathcal{L}_{r\textrm{-}cace}}(f^{*}) < \frac{1}{\psi} ( \mathcal{R}^{\eta}_{\mathcal{L}_{r\textrm{-}cace}}(f^{*}_{\eta}) - \mathcal{R}^{\eta}_{\mathcal{L}_{r\textrm{-}cace}}(f^{*}) - A\eta) \nonumber
	\end{align}
	or equivalently,
	\begin{align}
		\mathcal{R}^{\eta}_{\mathcal{L}_{r\textrm{-}cace}}(f^{*}_{\eta}) - \mathcal{R}^{\eta}_{\mathcal{L}_{r\textrm{-}cace}}(f^{*}) > \psi (\mathcal{R}_{\mathcal{L}_{r\textrm{-}cace}}(f^{*}_{\eta})-\mathcal{R}_{\mathcal{L}_{r\textrm{-}cace}}(f^{*}))+A\eta \nonumber
	\end{align}
	
	Since $f^{*}$ is the global minimizer of $\mathcal{R}_{\mathcal{L}_{r\text{-}cace}}(f)$ and $f^{*}_{\eta}$  is the global minimizer of $\mathcal{R}^{\eta}_{\mathcal{L}_{r\text{-}cace}}(f)$, we have 
	\begin{align}
		0\le \mathcal{R}_{\mathcal{L}_{r\textrm{-}cace}}(f^{*}_{\eta}) - \mathcal{R}_{\mathcal{L}_{r\textrm{-}cace}}(f^{*}) \nonumber < \frac{-A\eta}{\psi}
		=\frac{-A\eta (K-1)}{K(1-\eta)-1} \nonumber
	\end{align}
	and 
	\begin{align}
		A\eta<\mathcal{R}^{\eta}_{\mathcal{L}_{r\textrm{-}cace}}(f^{*}_{\eta}) - \mathcal{R}^{\eta}_{\mathcal{L}_{r\textrm{-}cace}}(f^{*}) \le 0 \nonumber
	\end{align}
	which concludes the proof.
\end{proof}

\begin{theorem}
	Under class-dependent label noise with $\eta_{ij}<1-\eta_{i},\forall j\ne i,\forall i,j \in [K]$, where $\eta_{ij}=p(\hat{y}=j|y=i), \forall j\ne i$ and $(1-\eta_{i}) = p(\hat{y}=i|y=i)$, if $\mathcal{R}_{\mathcal{L}_{r\text{-}cace}}(f^{*})=0$, then
	\begin{align}
		0\le\mathcal{R}_{\mathcal{L}_{r\text{-}cace}}^{\eta}(f^{*}) - \mathcal{R}_{\mathcal{L}_{r\text{-}cace}}^{\eta}(f^{*}_{\eta}) < G, \nonumber
	\end{align}
	where $G=A(1-K) \mathbb{E}_{\mathcal{D}(\bm{x},y)}(1-\eta_{y}) >0 $, $f^{*}$ and $f^{*}_{\eta}$ be the global minimizers of $\mathcal{R}_{\mathcal{L}_{r\textrm{-}cace}}(f)$ and $\mathcal{R}^{\eta}_{\mathcal{L}_{r\textrm{-}cace}}(f)$ respectively. 
	
\end{theorem}
\begin{proof}
	For asymmetric or class-dependent noise, we have 
	\begin{align}
		\mathcal{R}_{\mathcal{L}_{r\text{-}cace}}^{\eta}(f) &=\mathbb{E}_{\mathcal{D}_{\eta}(\bm{x},\hat{y})}[\mathcal{L}_{r\textrm{-}cace}(f(\bm{x}),\hat{y})] \nonumber \\
		&= \mathbb{E}_{\bm{x}}\mathbb{E}_{\mathcal{D}(y|\bm{x})}\Big[ (1-\eta_{y})\mathcal{L}_{r\textrm{-}cace}(f(\bm{x}),y) +\sum_{j\ne y}\eta_{yj}\mathcal{L}_{r\textrm{-}cace}(f(\bm{x}),j)\Big] \nonumber \\
		&= \mathbb{E}_{\mathcal{D}(\bm{x},y)}\Big[ (1-\eta_{y})\Big( \sum_{j=1}^{K}\mathcal{L}_{r\textrm{-}cace}(f(\bm{x}),j)  - \sum_{j\ne y}\mathcal{L}_{r\textrm{-}cace}(f(\bm{x}),j) \Big)\Big] \nonumber \\ & \quad \quad +\mathbb{E}_{\mathcal{D}(\bm{x},y)} \Big[\sum_{j\ne y}\eta_{yj}\mathcal{L}_{r\textrm{-}cace}(f(\bm{x}),j)\Big] \nonumber\\
		& <\mathbb{E}_{\mathcal{D}(\bm{x},y)} \Big[ (1-\eta_{y}) \Big( A(1-K) - \sum_{j\ne y}\mathcal{L}_{r\textrm{-}cace}(f(\bm{x}),j)\Big)\Big] \nonumber\\ & \quad \quad +   \mathbb{E}_{\mathcal{D}(\bm{x},y)} \Big[\sum_{j\ne y}\eta_{yj}\mathcal{L}_{r\textrm{-}cace}(f(\bm{x}),j)\Big] \nonumber \\
		&=A(1-K) \mathbb{E}_{\mathcal{D}(\bm{x},y)}(1-\eta_{y}) - \mathbb{E}_{\mathcal{D}(\bm{x},y)} \Big[ \sum_{j\ne y}(1-\eta_{y}-\eta_{yj}) \mathcal{L}_{r\textrm{-}cace}(f(\bm{x}),j) \Big]. \nonumber
	\end{align}
	On the other hand, we also have 
	\begin{align}
		\mathcal{R}_{\mathcal{L}_{r\text{-}cace}}^{\eta}(f) >  - \mathbb{E}_{\mathcal{D}(\bm{x},y)} \Big[ \sum_{j\ne y}(1-\eta_{y}-\eta_{yj}) \mathcal{L}_{r\textrm{-}cace}(f(\bm{x}),j) \Big] \nonumber
	\end{align}
	Hence, we obtain 
	\begin{align}
		&\mathcal{R}_{\mathcal{L}_{r\text{-}cace}}^{\eta}(f^{*}) - \mathcal{R}_{\mathcal{L}_{r\text{-}cace}}^{\eta}(f^{*}_{\eta}) < A(1-K) \mathbb{E}_{\mathcal{D}(\bm{x},y)}(1-\eta_{y})\nonumber  \\ & \quad \quad + \mathbb{E}_{\mathcal{D}(\bm{x},y)} \Big[ \sum_{j\ne y}(1-\eta_{y}-\eta_{yj}) \Big( \mathcal{L}_{r\textrm{-}cace}(f^{*}_{\eta}(\bm{x}),j) - \mathcal{L}_{r\textrm{-}cace}(f^{*}(\bm{x}),j) \Big) \Big] \nonumber
	\end{align}
	Next, we prove the bound. First, $(1-\eta_{y}-\eta_{yj}) > 0 $ as per the assumption that $\eta_{yj} < 1- \eta_{y}$. Second, our assumption has $\mathcal{R}_{r\textrm{-}cace}(f^{*})=0$, we have $\mathcal{L}_{r\textrm{-}cace}(f^{*}(\bm{x}),y)=0$. This is only satisfied iff $f^{*}_{j}(\bm{x})=1$ when $j=y$, and $f^{*}_{j}(\bm{x})=0$ when $j\ne y$. According to the definition of $\mathcal{L}_{r\textrm{-}cace}$, we have $\mathcal{L}_{r\textrm{-}cace}(f^{*}(\bm{x}),j) = -A\tau $, $\forall j\ne y$, and $\mathcal{L}_{r\textrm{-}cace}(f^{*}_{\eta}(\bm{x}),j) \le -A\tau $, $\forall j \in [K]$. We then obtain
	\begin{align}
		&\mathbb{E}_{\mathcal{D}(\bm{x},y)} \Big[ \sum_{j\ne y}(1-\eta_{y}-\eta_{yj}) \Big( \mathcal{L}_{r\textrm{-}cace}(f^{*}_{\eta}(\bm{x}),j)  - \mathcal{L}_{r\textrm{-}cace}(f^{*}(\bm{x}),j) \Big) \Big] \le 0 \nonumber
	\end{align}
	Therefore, we have
	\begin{align}
		\mathcal{R}_{\mathcal{L}_{r\text{-}cace}}^{\eta}(f^{*}) - \mathcal{R}_{\mathcal{L}_{r\text{-}cace}}^{\eta}(f^{*}_{\eta}) < A(1-K) \mathbb{E}_{\mathcal{D}(\bm{x},y)}(1-\eta_{y})\nonumber
	\end{align}	
	Since $f^{*}_{\eta}$ is the global minimizers of $\mathcal{R}^{\eta}_{\mathcal{L}_{r\textrm{-}cace}}(f)$, we have $\mathcal{R}_{\mathcal{L}_{r\text{-}cace}}^{\eta}(f^{*}) - \mathcal{R}_{\mathcal{L}_{r\text{-}cace}}^{\eta}(f^{*}_{\eta}) \ge 0$, which concludes the proof.
\end{proof}

\subsection{Comparison with existing noise-robust loss functions}
\label{sec:sample_wise_loss}
According to the definition in Section 3.5, we obtain the sample-wise
\begin{align}
	\mathcal{L}_{r\textrm{-}cace}&=-\sum_{k=1}^{K}\big(\tau (p(k|\bm{x})-q(k|\bm{x}))+q(k|\bm{x})\big)\log q(k|\bm{x}) \nonumber \\
	&= - \big(\tau (p(y|\bm{x})-  q(y|\bm{x})) + q(y|\bm{x})\big)\log q(y|\bm{x}) - {\sum}_{k\ne y}\big(\tau (p(k|\bm{x})-q(k|\bm{x}))+q(k|\bm{x})\big)\log q(k|\bm{x}) \nonumber \\ 
	&= - \big(\tau p(y|\bm{x})-\tau + 1\big)\log 1 - A\tau {\sum}_{k\ne y}p(k|\bm{x}) \nonumber\\ 
	&= - A\tau(1-p(y|\bm{x})). \text{ where } \tau\in (0,1) \text{ and } A \text{ is a negative constant }.
\end{align}
Similarly, we have sample-wise $\mathcal{L}_{mae}$ \cite{ghosh2017robust}, $\mathcal{L}_{rce}$ \cite{wang2019symmetric}, $\mathcal{L}_{gce}$ \cite{zhang2018generalized} and $\mathcal{L}_{tce}$ \cite{feng2020can} as follows
\begin{align}
	\mathcal{L}_{mae}=\sum_{k=1}^{K}|p(k|\bm{x})-q(k|\bm{x})|= (1-p(y|\bm{x}))+\sum_{k\ne y}p(k|\bm{x})=2(1-p(y|\bm{x})); \nonumber
\end{align}

\begin{align}
	\mathcal{L}_{rce}=-\sum_{k=1}^{K}p(k|\bm{x})\log q(k|\bm{x})=-p(y|\bm{x})\log 1 - \sum_{k\ne y}p(k|\bm{x})\log 0 = -A(1-p(y|\bm{x})); \nonumber
\end{align}

\begin{align}
	\mathcal{L}_{gce}=\sum_{k=1}^{K}q(k|\bm{x})\frac{1-p(k|\bm{x})^{\rho}}{\rho}=q(y|\bm{x})\frac{1-p(y|\bm{x})^{\rho}}{\rho}=\frac{1}{\rho}(1-p(y|\bm{x})^{\rho}), \rho \in (0,1]; \nonumber
\end{align}

\begin{align}
	\mathcal{L}_{tce}=\sum_{i=1}^{t}\frac{(1-p(y|\bm{x}))^{i}}{i}, t \in \mathbb{N}_{+}\ \text{denotes the order of Taylor Series.} \nonumber
\end{align}
We observe that when $\tau=1$ (even though it is impossible), $\mathcal{L}_{r\textrm{-}cace}$ is reduced to $\mathcal{L}_{rce}$. If $A=-2$ and $\tau=1$, $\mathcal{L}_{r\textrm{-}cace}$ is further reduced to $\mathcal{L}_{mae}$. Since confidence $\tau$ is various for different samples, $\mathcal{L}_{r\textrm{-}cace}$ is more like a dynamic version of $\mathcal{L}_{mae}$. As for $\mathcal{L}_{gce}$, $\lim_{\rho \rightarrow 0}\mathcal{L}_{gce}=\mathcal{L}_{ce}$ and $\mathcal{L}_{gce}=\frac{1}{2}\mathcal{L}_{mae}$ when $\rho=1$. Similarly, $\lim_{t \rightarrow \infty}\mathcal{L}_{tce}=\mathcal{L}_{ce}$ and $\mathcal{L}_{tce}=\frac{1}{2}\mathcal{L}_{mae}$ when $t=1$. Therefore, both $\mathcal{L}_{gce}$ and $\mathcal{L}_{tce}$ can be interpreted as the generalization of MAE and CE, which benefits the noise robust from MAE and training efficiency from CE. However, parameters $\rho$ and $t$ are fixed before training, so it is hard to tell what is the best parameter for the certain dataset. Instead, combined with $\mathcal{L}_\text{CAL}$, $\mathcal{L}_{r\textrm{-}cace}$ contains a dynamic confidence value $\tau$ for each sample that automatically learned from dataset, facilitating the learning from correctly-labeled samples.

\section{Algorithm}

Algorithm \ref{alg:1} provides detail pseudocode for CAR. Note that for Cosine Annealing learning rate scheduler, the condition line 8 becomes $e \ge E_{c}\ and\ \tau^{[i]} \ge \delta \ and\ e\%E_{p}==0$, where $E_{p}$ is the number of epochs in each period, we fix $E_{p}=10$ in all experiments. 

\begin{algorithm} [htb!]
	\nonumber
	\small
	\caption{Confidence adaptive regularization (CAR)}
	\label{alg:1}
	\KwIn{Deep neural network $\mathcal{N}_{\theta}$ with trainable parameters $\theta$; $\lambda$ is the parameter for penalty term $\mathcal{L}_{p}$; $\beta$ is the parameter for regularization term $\mathcal{L}_{r\text{-}cace}$; $E_{c}$ is the epoch that starts to estimate target; $\alpha$ is the momentum in target estimation; training set $D$, batch size $B$, total epoch $E_\text{max}$;}
	$\bm{t}=\bm{\hat{y}}$ \qquad \qquad\qquad\qquad\qquad\qquad\qquad\qquad\quad\quad \qquad\quad\Comment{Initialize the target by noisy labels}\;
	\For{$e=1,2,\dots,E_\text{max}$}
	{
		\textbf{Shuffle} $D$ into $\frac{|D|}{B}$ mini-batches \;
		\For{$n=1,2,\dots,\frac{|D|}{B}$}
		{
			\For{$i$ \text{in each mini-batch} $D_{n}$}
			{
				$\bm{p}^{[i]}=\mathcal{S}(\mathcal{N}_{\theta}(\bm{x}^{[i]}))\ \ $ \qquad\qquad\qquad\quad \Comment{Obtain model predictions}\;
				$\tau^{[i]} = \mathtt{sigmoid}(h^{[i]})$ \qquad\qquad\qquad\qquad\qquad \Comment{Obtain corresponding confidence}\;
				\If{$e \ge E_{c}$ and $\tau^{[i]} \ge \delta $ }{$\bm{t}^{[i]}=\alpha \bm{t}^{[i]}+(1-\alpha)\bm{p}^{[i]}$ \Comment{Target estimation}\;} 
				
			}
			\textbf{Calculate the loss} $\mathcal{L}_\text{CAR}= \mathcal{L}_{cace} + \lambda \mathcal{L}_{p} + \beta \mathcal{L}_{r\textrm{-}cace}= -\frac{1}{B}\sum_{i=1}^{B}(\bm{t}^{[i]})^{T}\log\big(\tau^{[i]}   
			(\bm{p}^{[i]}-\bm{t}^{[i]})+\bm{t}^{[i]}\big) -\frac{\lambda}{B}\sum_{i=1}^{B}\log(\tau^{[i]}) -\frac{\beta}{B}\sum_{i=1}^{B}\big(\tau^{[i]}(\bm{p}^{[i]}-\bm{t}^{[i]})+\bm{t}^{[i]}\big)^{T}\log(\bm{t}^{[i]})$ \;
			\textbf{Update} $\theta$ using stochatic gradient descent \; 
			
		}
	}
	\textbf{Output} $\theta$.
\end{algorithm}

\section{More results of label correction and confidence value}
\label{sec:more_label_correct}	
We report the label correction accuracy for various level of label noise on CIFAR-10 and CIFAR-100 in Table \ref{table:correct_acc}. Figure \ref{fig:more_confumatrix} displays the confusion matrix of corrected label w.r.t. the clean labels on CIFAR-10 with 60\% symmetric, 80\% symmetric and 40\% asymmetric label noise respectively. We also show the corrected labels for real-world datasets in Figure \ref{fig:webvision_correct} and Figure \ref{fig:clothing1M_correct}.

We report the confidence value for high level of label noise on CIFAR-10 in Figure \ref{fig:confidence_hm_60} and Figure \ref{fig:confidence_hm_80}. As we can see, the confidence values on the diagonal blocks remain relatively higher than those non-diagonal blocks.

\begin{table*}[htb!]
	\begin{center}
		\resizebox{1.0\textwidth}{!}{
			\centering
			\begin{tabular}{ p{25mm} c c c  c c c|c c c c c } 
				\toprule
				\multicolumn{2}{c}{\multirow{1}{*}{Dataset}} & \multicolumn{5}{c|}{CIFAR-10} &\multicolumn{5}{c}{CIFAR-100}\\\midrule
				\multicolumn{2}{c}{\multirow{2}{*}{Noise ratio}} & \multicolumn{4}{c}{symm} &\multicolumn{1}{c|}{asymm} & \multicolumn{4}{c}{symm}&\multicolumn{1}{c}{asymm} \\ 
				\multicolumn{2}{c}{}& 20\% & \multicolumn{1}{c}{40\%}&\multicolumn{1}{c}{60\%}&\multicolumn{1}{c}{80\%}& 40\% & \multicolumn{1}{c}{20\%}&\multicolumn{1}{c}{40\%}&\multicolumn{1}{c}{60\%}&\multicolumn{1}{c}{80\%}&\multicolumn{1}{c}{40\%}\\
				\midrule
				\multirow{1}{*}{Correction accuracy} &  & 97.3	 &95.1&91.1&81.1&93.8&92.6&86.4&76.5&40.4&87.1\\ 
				\bottomrule

			\end{tabular}
		}
	\end{center}
	\caption{Correction accuracy (\%) on CIFAR-10 and CIFAR-100 with various levels of label noise injected to training set.}
	\label{table:correct_acc}
\end{table*}

\begin{figure}[t]
	\begin{center}
		\includegraphics[width=1.0\linewidth]{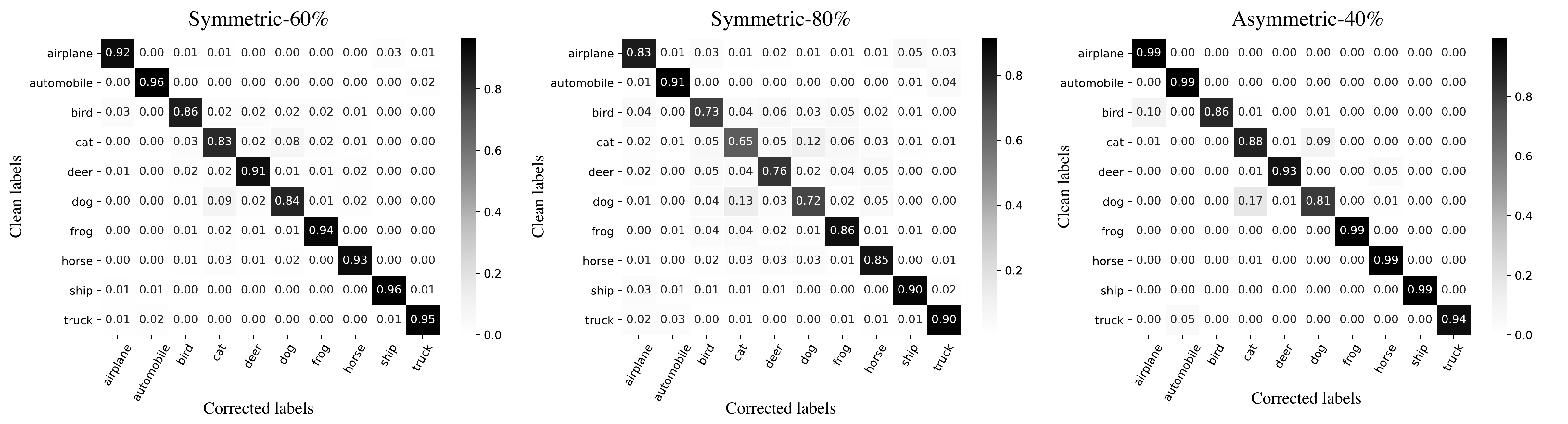}
	\end{center}
	\caption{Confusion matrix of corrected labels w.r.t clean labels on CIFAR-10 with 60\% symmetric, 80\% symmetric and 40\% asymmetric label noise respectively.}
	\label{fig:more_confumatrix}
\end{figure}

\begin{figure}[t]
	\centering
	\begin{minipage}[t]{0.45\textwidth}
		\centering
		\includegraphics[width=0.8\linewidth]{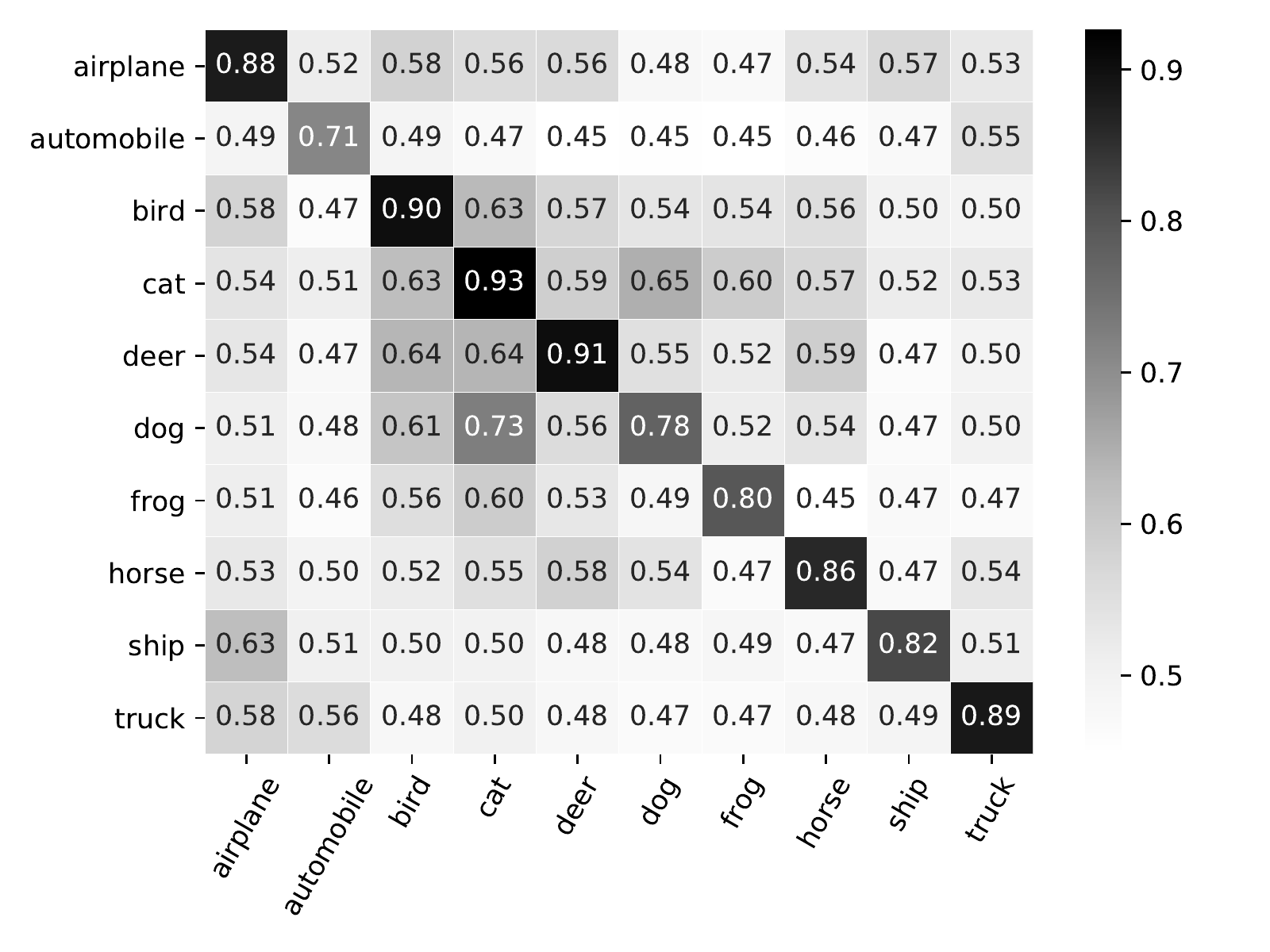}
		\caption{Average confidence values $\tau$ of false labels w.r.t clean labels on CIFAR-10 with 60\% symmetric label noise.}
		\label{fig:confidence_hm_60}
	\end{minipage}
	\hspace{0.5cm}
	\begin{minipage}[t]{0.45\textwidth}	
		\centering
		\includegraphics[width=0.8\linewidth]{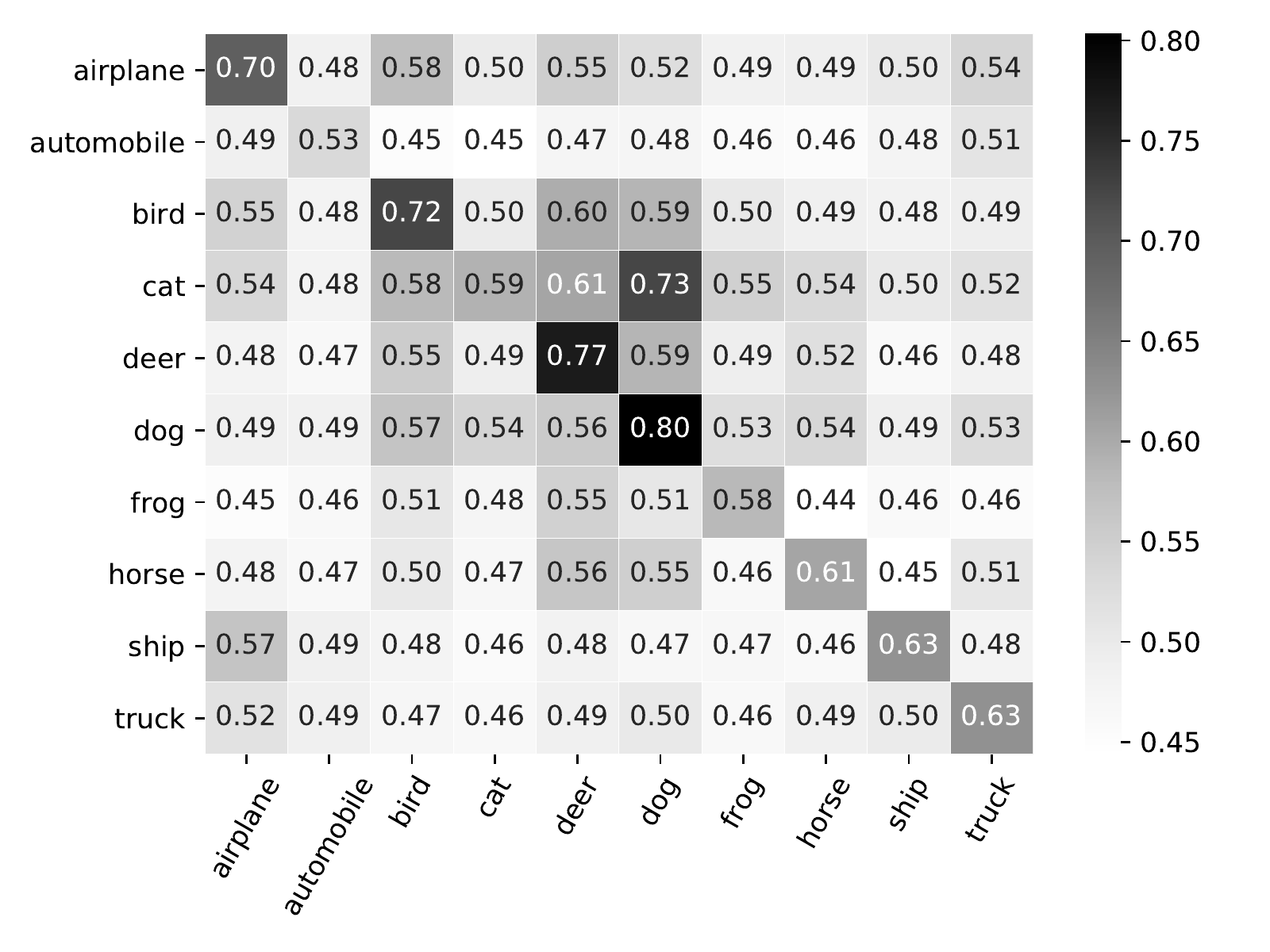}
		\caption{Average confidence values $\tau$ of false labels w.r.t clean labels on CIFAR-10 with 80\% symmetric label noise.}
		\label{fig:confidence_hm_80}
	\end{minipage}
	\vspace{-0.9em}
\end{figure}

\section{Detail description of experiments}
\label{sec:details_of_exp}
Source code for the experiments is available in the zip file. All experiments are implemented in PyTorch and run in a single Nvidia GTX 1080 GPU. For CIFAR-10 and CIFAR-100, we do not perform early stopping since we don't assume the presence of clean validation data. All test accuracy are recorded from the last epoch of training. For Clothing1M, it provides 50k, 14k, 10k refined clean data for training, validation and testing respectively. Note that we do not use the 50k clean data. We report the test accuracy when the performance on validation set is optimal. All tables of CIFAR-10/CIFAR-100 report the mean and standard deviation from 3 trails with different random seeds. As for larger datasets, we only perform a single trail.

\subsection{Dataset description and preprocessing}
\label{sec:preprocess}
The information of datasets are described in Table \ref{table:datasets}. CIFAR-10 and CIFAR-100 are clean datasets, we describe the label noise injection in Appendix \ref{apd:noise_inject}. Clothing1M consists of 1 million training images from 14 categories collected from online shopping websites with noisy labels generated from surrounding texts. Its noise level is estimated as 38.5\% \cite{song2019prestopping}. Following \cite{jiang2017mentornet,chen2019understanding}, we use the mini WebVision dataset which contains the top 50 classes from the Google image subset of WebVision, which results in approximate 66 thousand images. The noise level of WebVision is estimated at 20\% \cite{li2017webvision}.

As for data preprocessing, we apply normalization and regular data augmentation (i.e. random crop and horizontal flip) on the training sets of all datasets. The cropping size is consistent with existing works \cite{liu2020early,li2020dividemix}. Specifically, 32 for CIFAR-10 and CIFAR-100, 224 $\times$ 224 for Clothing 1M (after resizing to 256 $\times$ 256), and 227 $\times$ 227 for Webvision.

\begin{table}
	\caption{Detail information of experiment.}
	\begin{subtable}[t]{0.48\textwidth}
		\caption{Description of the datasets used in the experiments.}
		\begin{center}
			\resizebox{1.0\textwidth}{!}{
				\centering
				\begin{tabular}{c c c c  c c  c} 
					\toprule
					Dataset & \# of train & \# of val & \# of test & \# of classes & input size& Noise rate (\%) \\
					\midrule
					\multicolumn{7}{c}{\multirow{1}{*}{Datasets with clean annotation}} \\
					\midrule
					CIFAR-10 & 50K& - & 10K& 10 & 32 $\times$ 32& $\approx$ 0.0\\
					CIFAR-100 & 50K & - & 10K & 100 & 32 $\times$ 32 & $\approx$ 0.0\\
					\midrule
					\multicolumn{7}{c}{\multirow{1}{*}{Datasets with real world noisy annotation}} \\
					\midrule
					Clothing1M & 1M & 14K & 10K & 14 & 224 $\times$ 224 & $\approx$ 38.5\\
					Webvision 1.0 & 66K & - &2.5K & 50 & 256 $\times$ 256& $\approx$ 20.0\\	
					\bottomrule
				\end{tabular}
			}
		\end{center}
		\label{table:datasets}
	\end{subtable}
	\begin{subtable}[t]{0.48\textwidth}
		\caption{Description of the hyperparameters used in our approach.}
		\begin{center}
			\resizebox{1.0\textwidth}{!}{
				\centering
				\begin{tabular}{c c   } 
					\toprule
					Hyperparameter & Description  \\
					\midrule
					$\lambda$ & Control the strength of penalty loss in $\mathcal{L}_\text{CAL}$.\\
					$\beta$ & Control the strength of regularization term $\mathcal{L}_{r\text{-}cace}$.\\
					$E_{c}$ & The epoch starts to estimate target.\\
					$\alpha$ & The momentum in target estimation. \\
					$\delta$ & The threshold of confidence in target estimation.\\
					
					\bottomrule
				\end{tabular}
			}
		\end{center}
		\label{table:hyperparameter}
	\end{subtable}
\end{table}

\subsection{Simulated label noise injection}
\label{apd:noise_inject}
Since the CIFAR-10 and CIFAR-100 are initially clean, we follow \cite{tanaka2018joint,patrini2017making} for symmetric and asymmetric label noise injection. Specifically, symmetric label noise is generated by randomly flipping a certain fraction of the labels in the training set following a uniform distribution. Asymmetric label noise is simulated by flipping their class to another certain class according to the mislabel confusions in the real world. For CIFAR-10, the asymmetric noisy labels are generated by mapping \emph{truck} $\rightarrow$ \emph{automobile}, \emph{bird} $\rightarrow$ \emph{airplane}, \emph{deer} $\rightarrow$ \emph{horse} and \emph{cat} $\leftrightarrow$ \emph{dog}. For CIFAR-100, the noise flips each class into the next, circularly within super-classes.

\subsection{Training procedure}
\textbf{CIFAR-10/CIFAR-100}: We use a ResNet-34 and train it using SGD with a momentum of 0.9, a weight decay of 0.001, and a batch size of 64. The network is trained for 500 epochs for both CIFAR-10 and CIFAR-100. We use the cosine annealing learning rate \cite{loshchilov2016sgdr} where the maximum number of epoch for each period is 10, the maximum and minimum learning rate is set to 0.02 and 0.001 respectively. As for cross entropy with MultiStep learning rate scheduler  in Figure 1 and Figure 3 in the paper, we set the initial learning rate as 0.02, and reduce it by a factor of 10 after 100 and 200 epochs. The reason that we train the model 500 epochs in total is to fully evaluate whether the model will overfit mislabeled samples, which avoids the interference caused by early stopping \cite{li2020gradient} (i.e. the model may not start overfitting mislabeled samples when the number of training epochs is small, especially when learning rate scheduler is cosine annealing \cite{loshchilov2016sgdr}). 

\textbf{Clothing1M}: Following \cite{xiao2015learning,wang2019symmetric}, we use a ResNet-50 pretrained on ImageNet. We train the model with batch size 64. The optimization is done using SGD with a momentum 0.9, and weight decay 0.001. We use the same cosine annealing learning rate as CIFAR-10 except the minimum learning rate is set to 0.0001 and total epoch is 400. For each epoch, we sample 2000 mini-batches from the training data ensuring that the classes of the noisy labels are balanced.

\textbf{Webvision}: Following \cite{li2020dividemix,liu2020early}, we use an InceptionResNetV2 as the backbone architecture. All other optimization details are the same as for CIFAR-10, except for the weight decay (0.0005) and the batch size (32).

\subsection{Hyperparameters selection and sensitivity}
\label{sec:hyper_sensi}
Table \ref{table:hyperparameter} provides a detailed description of hyperparameters in our approach. We perform hyperparameter tuning via grid search: $\lambda=[0.5,10, 50]$, $\beta=[0.0,0.1,0.3,0.5]$, $E_{c}=[20,60,100]$, $\alpha=[0.7,0.9,0.99]$ and $\delta=[0,0,0.35,0.65,0.95]$. For CIFAR-10, the selected value are $\lambda=0.5$, $\beta=0.0$, $E_{c}=60$, $\alpha=0.9$ and $\delta=0.0$. For CIFAR-100 with 40\% asymmetric label noise, the selected value are $\lambda=10$, $\beta=0.1$, $E_{c}=20$, $\alpha=0.9$, $\delta=0.0$. For CIFAR-100 with 20\%/40\%/60\% symmetric label noise, we set $\lambda=10$, $\beta=0.1$, $E_{c}=60$, $\alpha=0.9$, $\delta=0.95$ and $\lambda=50$, $\beta=0.1$, $E_{c}=60$, $\alpha=0.9$, $\delta=0.0$ for 80\% symmetric label noise. For Webvision, we set  $\lambda=50$, $\beta=0.1$, $E_{c}=200$, $\alpha=0.9$, $\delta=0.0$. For Clothing1M, we set  $\lambda=50$, $\beta=0.1$, $E_{c}=60$, $\alpha=0.8$, $\delta=0.0$.

Figure \ref{fig:hyper_sensi} and Figure \ref{fig:hyper_sensi_cifar100} shows the hyperparameters sensitivity of CAR on CIFAR-10 and CIFAR-100 with 60\% symmetric label noise respectively. The coefficient of penalty loss $\lambda$ needs to be large than 0 to avoid trivial solution but also cannot be too large for CIFAR-10, avoiding neglecting $\mathcal{L}_{cace}$ term in the loss. As the CIFAR-10 is an easy dataset, no additional regularization requires by $\mathcal{L}_{r\text{-}cace}$ term. Therefore, the regularization coefficient $\beta$ should be 0 and large $\beta$ may cause model to underfit. The performance is robust to $E_{c}$ and $\alpha$, as long as the momentum $\alpha$ is large enough (e.g. larger than 0.7). The choice of confidence threshold $\delta$ depends on the difficulty of dataset. A larger $\delta$ will slow down the speed of target estimation but helps exclude ambiguous predictions with low confidence values. Overall, the sensitivity to hyperparameters is quite mild and the performance is quite robust, unless the parameter is set to be very large or very small, resulting in neglecting $\mathcal{L}_{cace}$ term or underfitting. We can observe the similar results of CIFAR-100 in Figure \ref{fig:hyper_sensi_cifar100}.


\begin{figure}[t]
	\begin{center}
		\includegraphics[width=1.0\linewidth]{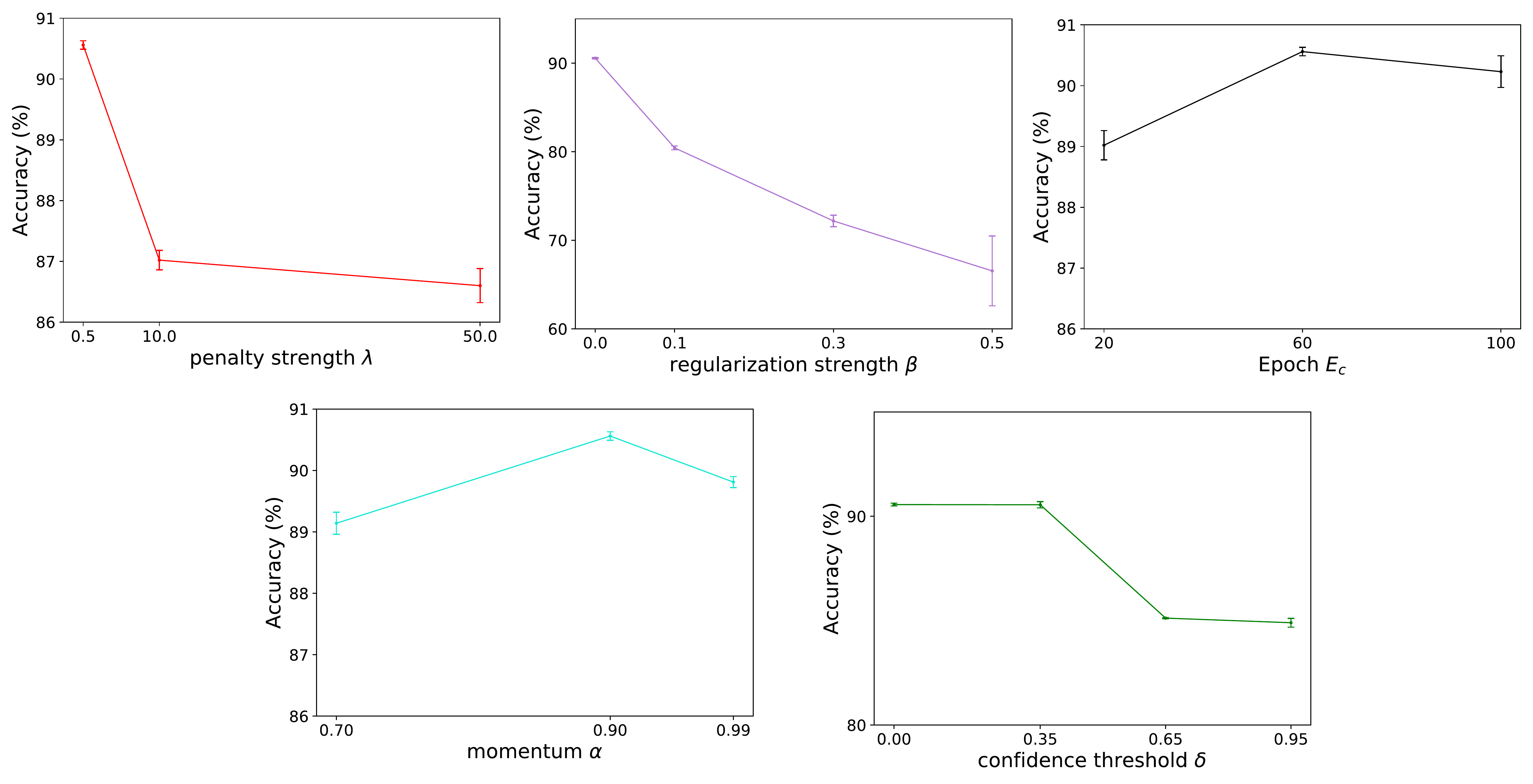}
	\end{center}
	\caption{Test accuracy on CIFAR-10 with 60\% symmetric label noise. The mean accuracy over three runs is reported, along with bars representing one standard deviation from the mean. In each experiment, the rest of hyperparameters are fixed to the values reported in Section \ref{sec:hyper_sensi}.}
	\label{fig:hyper_sensi}
\end{figure}

\begin{figure}[t]
	\begin{center}
		\includegraphics[width=1.0\linewidth]{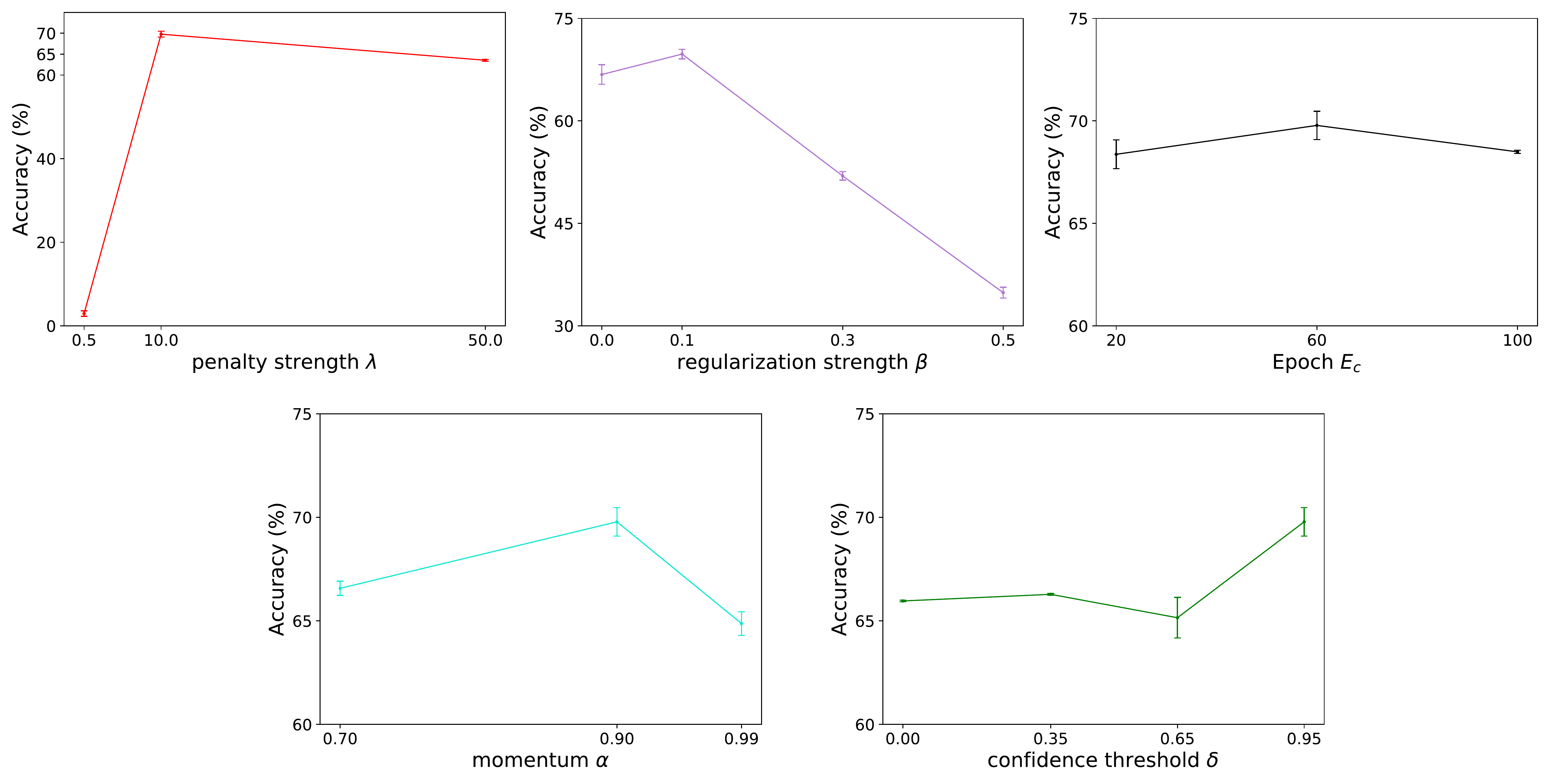}
	\end{center}
	\caption{Test accuracy on CIFAR-100 with 60\% symmetric label noise. The mean accuracy over three runs is reported, along with bars representing one standard deviation from the mean. In each experiment, the rest of hyperparameters are fixed to the values reported in Section \ref{sec:hyper_sensi}.}
	\label{fig:hyper_sensi_cifar100}
\end{figure}


\begin{figure}[t]
	\begin{center}
		\includegraphics[width=0.9\linewidth]{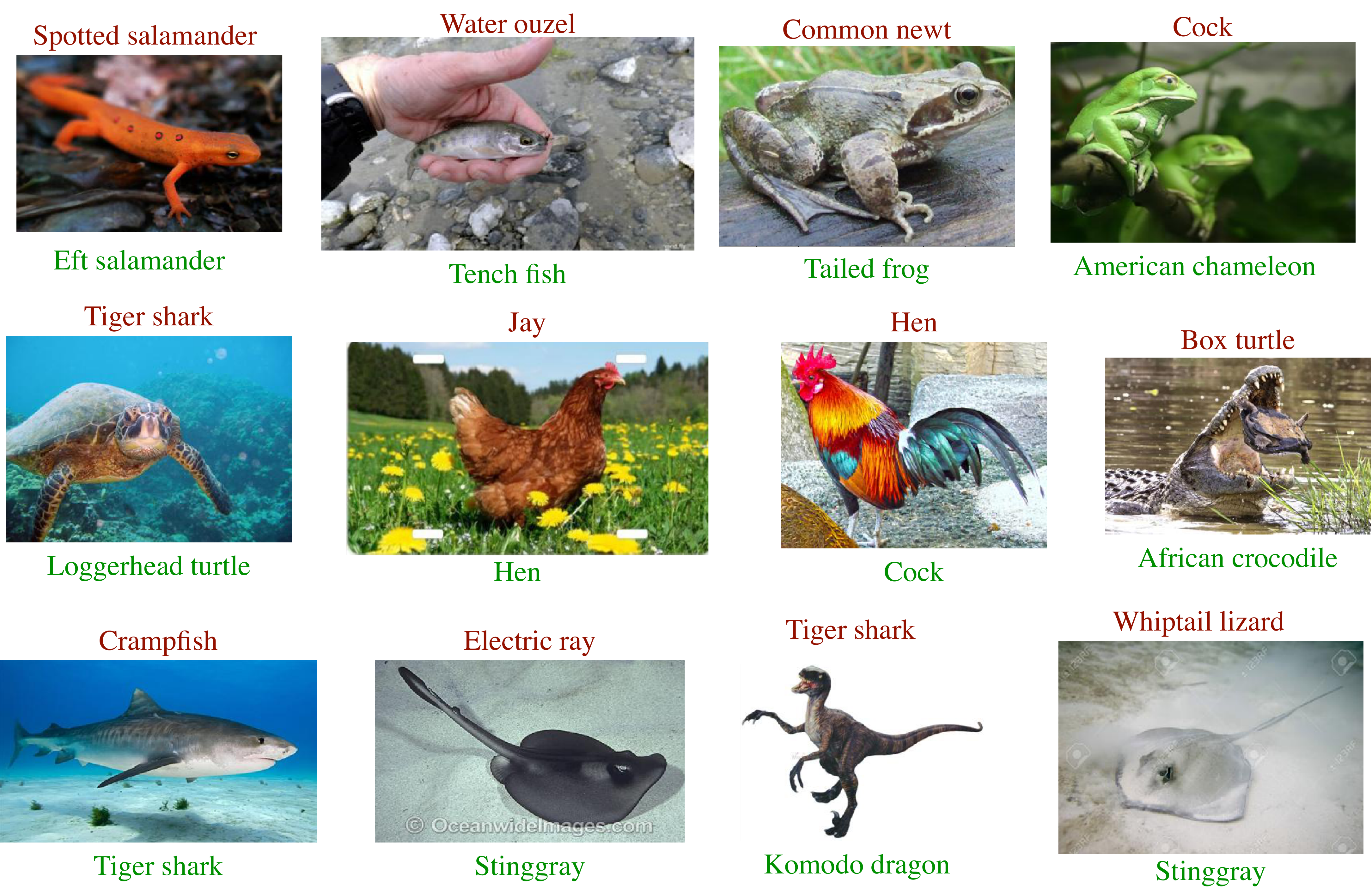}
	\end{center}
	\caption{Label correction of Webvision images. Given noisy labels are shown above in red and the corrected labels are shown below in green.}
	\label{fig:webvision_correct}
\end{figure}

\begin{figure}[htb!]
	\begin{center}
		\includegraphics[width=0.9\linewidth]{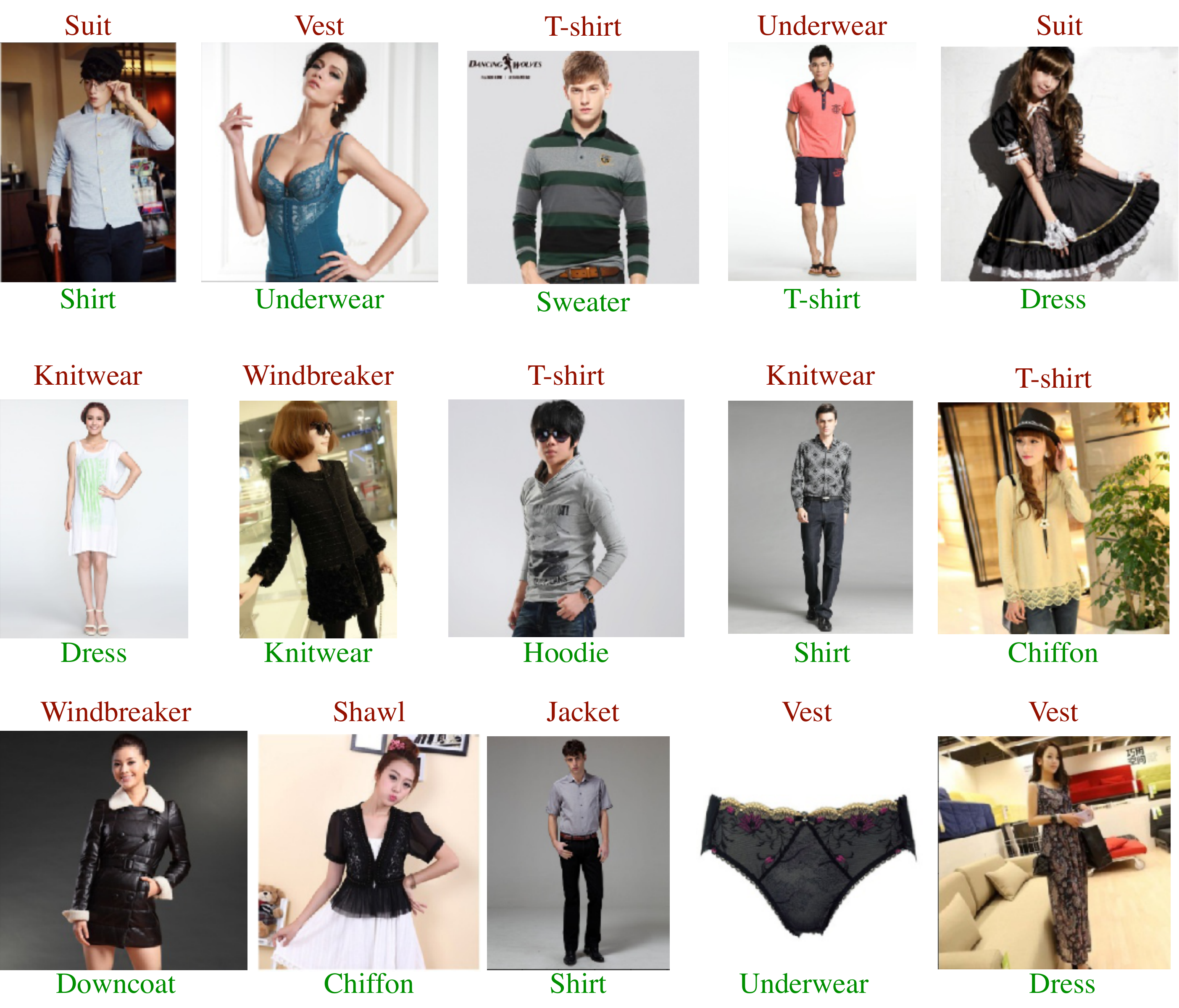}
	\end{center}
	\caption{Label correction of Clothing1M images. Given noisy labels are shown above in red and the corrected labels are shown below in green.}
	\label{fig:clothing1M_correct}
\end{figure}

\subsection{Performance with different target estimation strategies}
\label{apd:dif_stra}
We compare the performance of CAR with three strategies: 1) our strategy in Section 3.4. 2) temporal ensembling \cite{laine2016temporal} that adopted in ELR \cite{liu2020early}. 3) directly using the noisy labels $\hat{y}$ without target estimation. Table \ref{tab:dif_strategy} shows the results. As we can see, compared to CAR without target estimation, CAR with temporal ensembling does not improve much performance in easy cases (e.g. 40\% symmetric label noise), and it even gets worse performance in hard cases (e.g. 80\% symmetric label noise). However, CAR with our strategy achieves much better performance. We also conduct the experiments that use CE with different target estimation strategies. Surprisingly, CE with our strategy can achieve better performance to CAR in CIFAR-10 with 40\% asymmetric noise. However, the overall performance is worse than the performance of using CAR, due to the reason that CE will memorize noisy labels after early learning phase.



\begin{table}
	\caption{The test accuracy of CAR and CE with different target estimation strategies. All the following experiments use Cosine Annealing learning rate scheduler \cite{loshchilov2016sgdr}.}
	\begin{center}
		\resizebox{1\textwidth}{!}{
			\centering
			\begin{tabular}{ p{50mm} c c  c c c c c } 
				\toprule
				\multicolumn{2}{c}{\multirow{1}{*}{Dataset}} & \multicolumn{3}{c}{CIFAR-10} &\multicolumn{3}{c}{CIFAR-100}\\ \cmidrule(lr){3-5} \cmidrule(lr){6-8}
				\multicolumn{2}{c}{\multirow{1}{*}{Noise type}} & \multicolumn{2}{c}{symm} &\multicolumn{1}{c}{asymm} & \multicolumn{2}{c}{symm}&\multicolumn{1}{c}{asymm} \\   \cmidrule(lr){3-4} \cmidrule(lr){5-5} \cmidrule(lr){6-7} \cmidrule(lr){8-8}
				\multicolumn{2}{c}{\multirow{1}{*}{Noise ratio}}&  \multicolumn{1}{c}{40\%}&\multicolumn{1}{c}{80\%}& \multicolumn{1}{c}{40\%} &\multicolumn{1}{c}{40\%}&\multicolumn{1}{c}{80\%}&\multicolumn{1}{c}{40\%}\\
				\midrule
				
				\multirow{1}{*}{CAR with our strategy} & &\textbf{93.49} $\pm$ \textbf{0.07} &\textbf{80.98} $\pm$ \textbf{0.27}&\textbf{92.09} $\pm$ \textbf{0.12}&\textbf{75.38} $\pm$ \textbf{0.08}&\textbf{38.24} $\pm$ \textbf{0.55}&\textbf{74.89} $\pm$ \textbf{0.20}\\
				\multirow{1}{*}{CAR with temporal ensembling \cite{laine2016temporal}} & &89.52 $\pm$ 0.30 &64.07 $\pm$ 2.04&80.52 $\pm$ 2.21&70.80 $\pm$ 0.38&10.28 $\pm$ 1.67&63.91 $\pm$ 1.65\\
				\multirow{1}{*}{CAR w/o target estimation} & &89.47 $\pm$ 0.50&76.91 $\pm$ 0.22&88.23 $\pm$ 0.22&69.91 $\pm$ 0.21&31.33 $\pm$ 0.38&55.68 $\pm$ 0.17\\
				
				\midrule
				\multirow{1}{*}{CE with our strategy} & &92.64 $\pm$ 0.21&75.51 $\pm$ 0.38&\textbf{92.21} $\pm$ \textbf{0.11}&68.53 $\pm$ 0.47&32.36 $\pm$ 0.44&73.01 $\pm$ 0.90\\
				\multirow{1}{*}{CE with temporal ensembling \cite{laine2016temporal}} & &92.12 $\pm$ 0.16&72.87 $\pm$ 1.98&89.71 $\pm$ 1.43&70.45 $\pm$ 0.22&9.34 $\pm$ 0.78&66.38 $\pm$ 0.57\\
				\multirow{1}{*}{CE w/o target estimation} & &78.26 $\pm$ 0.74&56.42 $\pm$ 2.49&86.55 $\pm$ 1.06&46.34 $\pm$ 0.56&11.55 $\pm$ 0.35&48.86 $\pm$ 0.04\\
				\bottomrule	
			\end{tabular}
		}
	\end{center}
	\label{tab:dif_strategy}
\end{table}



\end{document}